\pdfoutput=1
\documentclass[10pt]{article}

\PassOptionsToPackage{sort&compress}{natbib}

\usepackage{thuidm}
\setcitestyle{numbers,square,comma}

\usepackage{mathtools}
\usepackage{mathrsfs}
\usepackage{eucal}          
\usepackage{nicefrac}
\usepackage{multirow}
\usepackage{multicol}
\usepackage{colortbl}
\usepackage{wrapfig}
\usepackage{float}
\usepackage[ruled,vlined]{algorithm2e}

\definecolor{myhighlight}{RGB}{220,240,255}

\title{Bridging Risk Approximation Gaps in Model Predictive\\
       Task Sampling via In-Context Modeling}

\author{%
  \mbox{Jiarong Wen\aff{1,$\dagger$}}, \mbox{Qi Tao\aff{1,$\dagger$}},
  \mbox{Kaiyu Zhang\aff{1,$\dagger$}}, \mbox{Yun Qu\aff{2}}, \mbox{Lizhou Cai\aff{2}},\\
  \mbox{Heming Zou\aff{2}}, \mbox{Yuhang Jiang\aff{2}}, \mbox{Qi Wang\aff{2}}%
}

\affiliation{%
  \aff{1}College of Science, National University of Defense Technology\quad
  \aff{2}Department of Automation, Tsinghua University\\[1.5pt]
  \aff{$\dagger$}Equal contribution\\[4pt]
  {\sffamily\bfseries\color{thupurpledark}Accepted at NeurIPS 2026}%
}

\datestamp{October 10, 2026}
\correspondence{Qi Wang at
  \href{mailto:cheemswang@mail.tsinghua.edu.cn}{cheemswang@mail.tsinghua.edu.cn}}

\labwebsite[www.thuidm.com]{https://www.thuidm.com/}

\begin{document}

\begin{abstract}
Reinforcement learning (RL) is essential for adaptive decision-making, e.g., robust Meta-RL and domain randomization (DR), and efficient post-training of large language models (LLMs), where tasks are commonly structured as Markov decision processes or prompts.
A central bottleneck in all these settings is the cost of agent-environment interactions required for task selection and policy optimization.
Model predictive task sampling (MPTS) has emerged as a promising approach to improve sample efficiency by querying informative tasks via a risk predictive model (RPM) trained on optimization history.
However, existing RPMs suffer from two failure modes: (i) sparse optimization histories, which limit RPMs' generalization across the task space, and (ii) a one-step lag in risk approximation, which causes biased difficulty estimation as the policy evolves.
This work formalizes these failure modes, derives lagged-risk decompositions that expose the resulting temporal bias, and recasts risk prediction as an in-context sequence modeling problem.
The developed RPM integrates a one-step look-ahead mechanism with a temporal-weighted sliding window to mitigate data non-stationarity and scarcity, without requiring additional rollouts.
Empirically, our approach yields more reliable difficulty estimates and consistent performance gains across robust Meta-RL, DR, and prompt curriculum RL post-training of LLMs.
\end{abstract}

\maketitle

\section{Introduction}
Over the past decade, deep reinforcement learning (RL) has achieved remarkable progress in complex sequential decision-making~\citep{sutton2018reinforcement,mnih2015human,hokoff,cfcql,wang2022model}. 
This work focuses on two settings where RL operates over structured task distributions. 
(i) In robotics and physical control, agents must adapt to new Markov decision processes (MDPs) in zero-shot or few-shot scenarios under safety constraints~\citep{carpin2016risk,rempe2022generating}. 
Robust domain randomization (DR)~\citep{muratore2018domain} and robust meta-reinforcement learning (Meta-RL)~\citep{finn2017model} address this by optimizing over MDP distributions and prioritizing risky environments. 
(ii) In foundation model post-training, RL has become central to incentivizing reasoning in large language models (LLMs)~\citep{ouyang2022training,achiam2023gpt}, where identifying informative prompts and rollouts is critical for stable optimization~\citep{yu2025dapo,qu2026can,gao2026prompt,gps}. 
Optimization over prompts further requires policies to explore complex logical search spaces through long-horizon token generation.

\paragraph{Evaluation Bottlenecks in Task Prioritization.}
Despite their different objectives, both settings share the same bottleneck: expensive policy evaluation. 
In robust DR and Meta-RL, assessing task risk requires extensive agent-environment trajectories on selected tasks~\citep{haarnoja2018soft,wang2024simple,lv2024theoretical,wang2025robust}; in LLM post-training, each update requires long Chain-of-Thought (CoT) rollouts~\citep{wei2023cot}, making prompt difficulty evaluation costly~\citep{yu2025dapo,qu2026can,lpo}. 
Thus, exhaustive task evaluation is impractical, motivating efficient task prioritization.

\paragraph{Conditional Agent-Environment Interactions via Model Predictive Strategies.}
Model predictive task sampling (MPTS)~\citep{wang2025modelpredictivetasksampling} and its variants~\citep{qu2026can,gao2026prompt,qu2025fast,dps} reduce unnecessary interactions by using optimization histories, e.g., task or prompt identifiers paired with losses, to train an RPM that approximates task risk online. 
This enables \textit{conditional interactions}: policy updates target selected tasks without exhaustive rollouts. 
However, RPM fidelity remains underexplored despite directly governing task selection and policy evolution. 
Vanilla RPMs~\citep{wang2025modelpredictivetasksampling,qu2025fast} suffer from two key limitations: (i) \textit{temporal lag bias}, where outdated policy snapshots misrepresent the current risk landscape; and (ii) \textit{sparse-history generalization}, where limited interaction data restricts generalization across the task space. 
Analogous issues appear in prompt curriculum learning for LLM post-training, e.g., PCL~\citep{gao2026prompt}, where temporal lag and sparse evaluations distort curricula and reduce sample efficiency.

\paragraph{Enhancing Risk Predictive Models via In-Context Sequence Modeling.}
We analyze vanilla RPMs and quantify extrapolation errors caused by temporal lag and data scarcity under small task batches. 
To address these issues, we propose two remedies: (i) a \textit{one-step look-ahead mechanism} that conditions risk prediction on recent optimization signals to reduce extrapolation bias; and (ii) a \textit{temporal-weighted sliding window} that augments training with recent task trajectories while exponentially downweighting stale data. 
Together, they instantiate an in-context model of the evolving task risk function, improving prediction under non-stationary policy dynamics with negligible additional rollout cost.

\paragraph{Primary Contributions and Empirical Findings.}
Focusing on RPM design for the MPTS family, our contributions are two-fold:
\begin{enumerate}
    \item \textbf{Theoretical Analysis of MPTS Vulnerabilities:}
    We derive a diagnostic lagged-risk decomposition that exposes finite-sample reconstruction error and a policy-update-induced $\mathcal O(\eta)$ mismatch, explaining why interpolation-based RPMs can destabilize task-risk rankings.
    \item \textbf{In-Context Modeling of the Evolving Risk Function:} We introduce in-context RPM (ICRPM) to characterize optimization histories, with one-step look-ahead prediction as the core design and a temporal-weighted sliding window as lightweight data reuse to reduce extrapolation bias and alleviate data scarcity without modifying the policy optimization loop.
\end{enumerate}
Experiments within the MPTS framework show that ICRPM improves sample efficiency and tail-risk robustness on robust DR and Meta-RL benchmarks, and accelerates curriculum-based RL post-training for LLMs, demonstrating better generalization across continuous-control and LLM reasoning tasks. 

\section{Research Background}

\paragraph{Notation and Setup.}
We consider a unified task space $\mathcal{T}$ that encompasses both continuous control and LLM post-training regimes. Each task $\boldsymbol{\tau} \in \mathcal{T}$ is formalized as a MDP \citep{sutton2018reinforcement}. In DR and Meta-RL, $\boldsymbol{\tau}$ is represented as a continuous descriptor in $\mathbb{R}^d$, whereas in reinforcement learning with verifiable rewards (RLVR), $\boldsymbol{\tau}$ corresponds to a discrete reasoning prompt.  
At optimization step $t$, the agent is parameterized by $\boldsymbol{\theta}_t \in \Theta$, representing either a control policy or an autoregressive model $\pi_{\boldsymbol{\theta}_t}$.

Over the course of training adaptive decision-makers or RL finetuning LLMs, we sample a batch of tasks $\mathcal{T}_t^{B} = \{\boldsymbol{\tau}_{t,i}\}_{i=1}^{B} \subset \mathcal{T}$ for the agent to interact.
Let $\mathcal{L}:\Theta\times\mathcal{T}\to\mathbb{R}$ denote the policy-dependent risk, and define 
$\ell_t(\boldsymbol{\tau}):=\mathcal{L}(\boldsymbol{\theta}_t,\boldsymbol{\tau})$ as the risk of policy $\pi_{\boldsymbol{\theta}_t}$ on task $\boldsymbol{\tau}$ at iteration $t$. 
Concretely, $\ell_t(\boldsymbol{\tau})$ can be the negative episodic return in control settings,
$\ell_t(\boldsymbol{\tau})=-\mathbb{E}_{\pi_{\bm\theta_t}}\!\left[\sum_{i=0}^{H-1}\bar{\lambda}^{i}r_{t,i}^{\boldsymbol{\tau}}\right]$
under post-adapted policies $\pi_{\bm\theta_t}$, or the expected penalty in RLVR, e.g.,
$\ell_t(\boldsymbol{\tau}) = 1 - \mathbb{E}[r_t^{\boldsymbol{\tau}}]$.
In the MPTS setup \citep{wang2025modelpredictivetasksampling,qu2025fast,qu2026can,gao2026prompt}, the per-iteration optimization batch is $H_t = \{(\boldsymbol{\tau}_{t,i}, 
\ell_{t}(\bm{\tau}_{t,i}))\}_{i=1}^{B}$, and the accumulated form $H_{0:t} 
= \bigcup_{k=0}^{t} H_k$ tracks the temporal optimization history. 
We further denote the space of complete historical 
trajectories by $\mathcal{H}$ and define the RPM as a functional 
$\hat{\ell}_t : \mathcal{T} \times \mathcal{H} \rightarrow \mathbb{R}$, which approximates the non-stationary task risk landscape conditioned on $H_{0:t}$.

\paragraph{General MPTS Pipeline.}
MPTS~\citep{wang2025modelpredictivetasksampling,qu2025fast,gao2026prompt}
executes a \textit{predict-then-select} pipeline, which can instantiate both
active prompt selection and robust task sampling. 
For prompt curriculum learning in RLVR~\citep{qu2026can,gao2026prompt},
the RPM predicts the sampled success rate $\hat{\rho}_{\bm{\tau}}^t$, and prompts
with intermediate difficulty, i.e., those close to a target threshold $\rho^*$, typically
around $0.5$, are preferred~\citep{bae2026online}. 
For robust task sampling, the same
rule is applied to the predicted risk by replacing
$\hat{\rho}_{\bm{\tau}}^t$ with $\hat{\ell}_t(\bm{\tau})$ and setting
$\rho^*=\max_{\bm{\tau}\in\mathcal{T}_{t}^{\hat{B}}}\hat{\ell}_t(\bm{\tau})$,
which encourages near-worst-case task selection.

At each iteration, MPTS proceeds in three steps:
\textcircled{1} \emph{Online RPM update.} The RPM is refined with the latest
history $H_{t-1}$ by maximizing the likelihood of risk prediction~\citep{gao2026prompt} or its approximation~\citep{wang2025modelpredictivetasksampling}, e.g., the ELBO in Eq.~\eqref{elbo}.
\textcircled{2} \emph{Amortized evaluation.}  Let $\hat{B}$ denote candidate-pool size, a candidate pool
$\mathcal{T}_{t}^{\hat{B}}\sim p(\bm{\tau})$ is sampled, and the forward
passes through the RPM produce predictive risk for all candidate tasks.
\textcircled{3} \emph{Subset selection.} The training task batch is constructed by
selecting $B$ tasks with predicted risk closest to the target threshold:
\begin{subequations}
\label{eq:mpts_pipeline}
\begin{align}
\text{Amortized Train RPM:}\quad
&\max_{\bm{\psi}\in\Psi,\,\bm{\phi}\in\Phi}
\mathcal{L}_{\mathrm{RPM}}(\bm{\psi},\bm{\phi};H_{t-1}),
\\
\text{Online Task Evaluation:}\quad
&p_{\bm{\psi}}\!\left(
\ell \mid \hat{\bm{\tau}}_i,H_{t-1};\bm{\theta}_t
\right)
\xrightarrow{\mathrm{MC}}
\left\{
\hat{\rho}_{\hat{\bm{\tau}}_i}^{t},\,m(\ell_i),\,\sigma(\ell_i)
\right\}_{i=1}^{\hat{B}},
\label{eq:mpts_pipeline_b}
\\
\text{Active Subset Selection:}\quad
&\mathcal{T}_{t}^{B}
=
\operatorname*{arg\,top-}B_{\bm{\tau} \in \mathcal{T}_{t}^{\hat{B}}}
\left(
-\big\| \hat{\rho}_{\bm{\tau}}^{t} - \rho^* \big\|_2^2
\right)
\end{align}
\end{subequations}
In Eq.~\eqref{eq:mpts_pipeline_b}, $m(\ell_i)$ and $\sigma(\ell_i)$ are the predictive mean and standard deviation of the task outcome estimated via Monte Carlo sampling, which are then mapped to the acquisition scalar $\hat{\rho}_{\hat{\bm{\tau}}_i}^{t}$ to quantify task priority.
With the queried batch in Eq.~(\ref{eq:mpts_pipeline}c), the policy model $\pi_{\bm\theta}$ is accordingly updated in the next iteration, e.g., $\bm\theta_{t+1}=\bm\theta_{t}-\eta\frac{1}{B}\sum_{b=1}^{B}\nabla_{\bm{\theta}}\ell_{t}(\bm\tau_b)$ in DR.
By repeating these steps for $T$ rounds, MPTS achieves either prompt curriculum learning in LLMs or robust Meta-RL and DR.

\paragraph{RPM and Variational Inference.}
Existing MPTS frameworks typically instantiate the RPM as a VAE-based latent generative predictor~\citep{kingma2013auto}, where a latent state $\mathbf{z}_t$ summarizes historical task-risk observations. 
Using streaming variational inference~\citep{variationalinference,rezende2014stochastic}, the RPM maximizes the per-step ELBO:
\begin{equation}
\label{elbo}
\mathcal{G}_{\text{ELBO}}(\boldsymbol{\psi}, \boldsymbol{\phi})
= \mathbb{E}_{q_{\boldsymbol{\phi}}(\mathbf{z}_t \vert H_t)}
\left[ \sum_{i=1}^{B} \log p_{\boldsymbol{\psi}}(\ell_{t}(\bm{\tau}_{t,i}) \vert \boldsymbol{\tau}_{t,i}, \mathbf{z}_t) \right]
- \beta\mathrm{KL}\left(
q_{\boldsymbol{\phi}}(\mathbf{z}_t \vert H_t)
||
q_{\bar{\boldsymbol{\phi}}}(\mathbf{z}_t \vert H_{t-1})
\right),
\end{equation}
where $\bar{\boldsymbol{\phi}}$ represents the parameters from the previous iteration, kept fixed to provide a stable prior.
For task evaluation, RPMs use Monte Carlo forward passes to estimate
$
p(\ell \vert \boldsymbol{\tau}, H_{0:t}; \boldsymbol{\theta}_t)
\approx
\mathbb{E}_{q_{\boldsymbol{\phi}}(\mathbf{z}_t \vert H_t)}
\left[
p_{\boldsymbol{\psi}}(\ell \vert \boldsymbol{\tau}, \mathbf{z}_t; \boldsymbol{\theta}_t)
\right].
$
These predictions are then used by the acquisition rule in Eq.~(\ref{eq:mpts_pipeline}c) for robust~\citep{wang2025modelpredictivetasksampling,qu2025fast} or curriculum optimization~\citep{qu2026can,gao2026prompt}. 
Thus, vanilla RPMs rely on generalization at $\bm\theta_t$ to approximate task difficulty after the policy update to $\bm\theta_{t+1}$, creating the temporal mismatch analyzed next.

\section{From Theoretical Investigation to Practical Enhancement}

\begin{figure}[t]
    \centering
    \begin{minipage}[c]{0.53\textwidth}
        \centering
        \includegraphics[width=\linewidth]{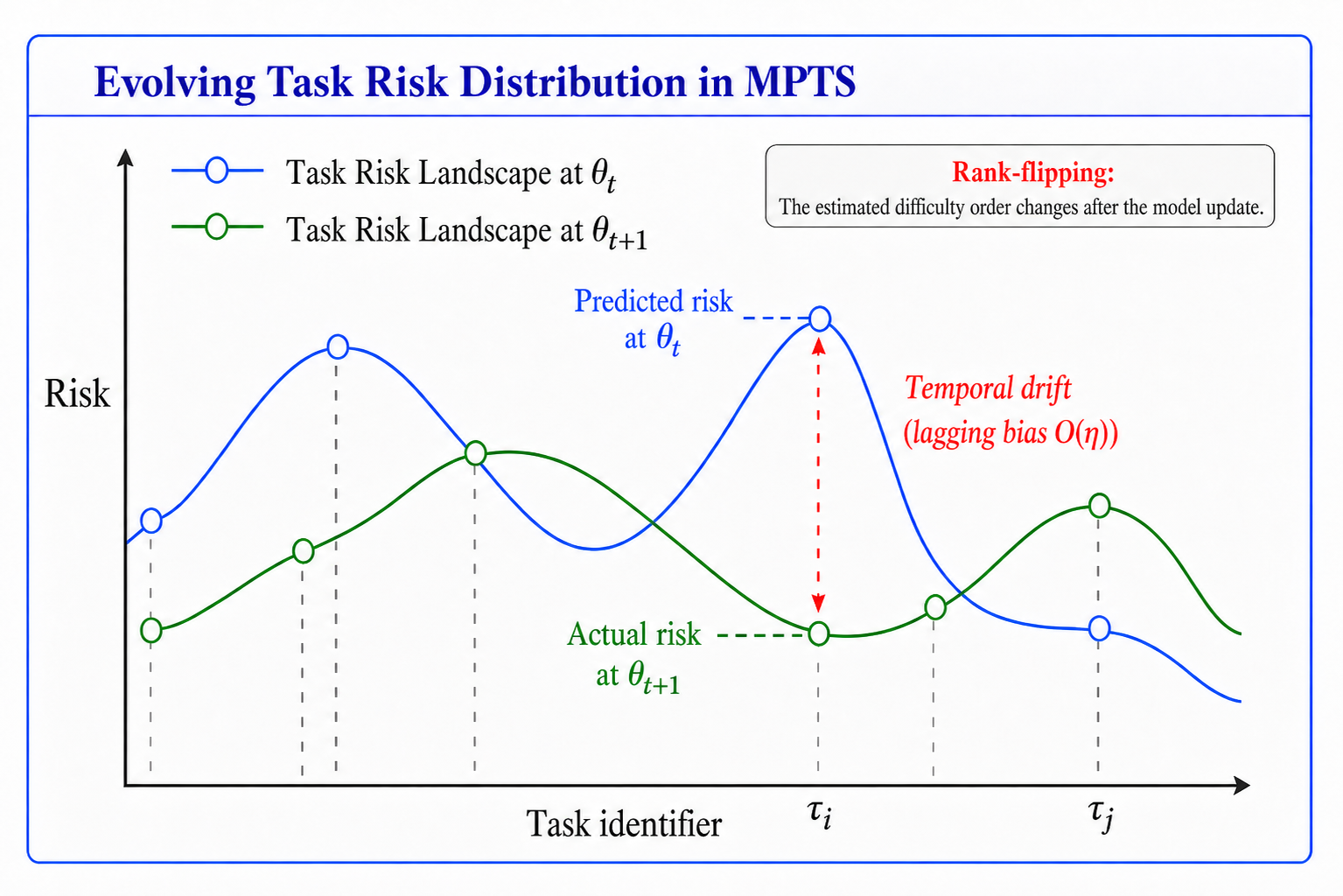}
    \end{minipage}\hfill
    \begin{minipage}[c]{0.43\textwidth}
        \caption{\textbf{Task Risk Distribution Shift over Iterations.}
        Using the risk generalization results at $\bm\theta_{t}$ tends to cause the
        \textit{rank-flipping} at $\bm\theta_{t+1}$, i.e.,
        $\text{sgn}[\ell_{t}(\bm\tau_i)-\ell_{t}(\bm\tau_j)]\neq
         \text{sgn}[\ell_{t+1}(\bm\tau_i)-\ell_{t+1}(\bm\tau_j)]$.}
        \label{fig:mpts_lag}
    \end{minipage}
\end{figure}

This section analyzes the theoretical bottlenecks of RPMs, introduces the In-Context Predictive Model (ICRPM), and discusses its implementation within MPTS.

\subsection{Vulnerabilities of Vanilla RPM}
\label{sec:theory_vanilla}
Although MPTS provides a sample-efficient surrogate for task risk estimation, the vanilla RPM overlooks a common failure mode, i.e., the continual evolution of both the policy and the task risk landscape induces rank-flipping (Fig.~\ref{fig:mpts_lag}) in task risk that cannot be corrected post-hoc.

\paragraph{Instability of Task Risk Rankings.}
MPTS relies on accurate task-risk ranking, but task risk is policy-dependent: after updating from $\bm\theta_t$ to $\bm\theta_{t+1}$, the ordering induced by $\ell_t(\boldsymbol{\tau})$ may differ from that of $\ell_{t+1}(\boldsymbol{\tau})$. 
As shown in Fig.~\ref{fig:mpts_lag}, this policy-induced drift creates a non-stationary difficulty landscape. 
Prior analysis~\citep{wang2025modelpredictivetasksampling} shows that ranking stability requires sufficiently small learning rates and smooth loss landscapes, conditions often violated in practice. 
Hence, even a perfectly accurate RPM at $\bm\theta_t$ may fail to preserve the correct ranking at $\bm\theta_{t+1}$, which we call the \textit{task risk distribution shift}.

\paragraph{Lagged Risk Prediction and Reconstruction Bound.}
The vanilla RPM, a VAE-based risk generative model~\citep{wang2025modelpredictivetasksampling,qu2025fast}, is trained on the latest history batch $H_{t}$ of size $B$ in each per-step update, where supervision targets are evaluated at the latest yet lagged parameter $\bm\theta_{t}$. 
The induced structured risk distribution shift cannot be resolved by increasing sample size alone: 
even under standard capacity and smoothness assumptions, the prediction error includes a temporal mismatch term induced by the policy update from $\bm\theta_t$ to $\bm\theta_{t+1}$.

\begin{theorem}[Per-step Lagged Prediction Bound for Vanilla RPMs] 
\label{thm:vanilla_bound} 
Assume that $\ell(\bm{\tau})=\mathcal{L}(\bm{\theta},\bm{\tau})\in[0,M]$ for all
$\bm{\theta}\in\Theta$, $\bm{\tau}\in\mathcal T$, and that
$\ell(\bm{\tau})$ is $L$-Lipschitz in $\bm{\theta}$. The policy update uses a normalized average loss computed by i.i.d. samples and
$\|\nabla_{\bm{\theta}}\ell_t(\bm{\tau}_i)\|\le G$.
Let $\hat\ell_{\bm{\psi},\bm{\phi},t}(\bm{\tau})$ denote the predictive risk
function induced by the vanilla RPM at iteration $t$. Define
$
    \epsilon_{\mathrm{app}}^{\mathrm{rec}}:=
    \inf_{\bm{\psi},\bm{\phi}}
    \mathbb E_{\bm{\tau}\sim p}
    \bigl[
        \bigl(\hat\ell_{\bm{\psi},\bm{\phi},t}(\bm{\tau})
        -\ell_t(\bm{\tau})\bigr)^2
    \bigr].
$
Then, for any $\delta\in(0,1)$, with probability at least $1-\delta$ over the
sampling of $H_t$,
\begin{equation}
\label{eq:vanilla_bound}
    \mathbb E_{\bm{\tau}\sim p}
    \bigl[|\hat\ell_t(\bm{\tau})-\ell_{t+1}(\bm{\tau})|\bigr]
    \le
    \underbrace{
    \sqrt{\epsilon_{\mathrm{app}}^{\mathrm{rec}}
    +\epsilon_{\mathrm{gen}}^{\mathrm{rec}}(B,\delta)}
    }_{\textcircled{\small{1}}\text{Reconstruction Approximation Error}}
    \qquad
    +
    \underbrace{
    L\eta G
    }_{\textcircled{\small{2}}\text{Temporal Lagging Error}} .
\end{equation}
If the reconstruction loss class
$
    \mathcal G^{\mathrm{rec}}
    =
    \{(\bm{\tau},y)\mapsto
    (\hat\ell_{\bm{\psi},\bm{\phi},t}(\bm{\tau})-y)^2:
    \bm{\psi}\in\Psi,\bm{\phi}\in\Phi\}
$
has Rademacher complexity $\mathfrak R_B(\mathcal G^{\mathrm{rec}})$ and is bounded by $M^2$, then
$\epsilon_{\mathrm{gen}}^{\mathrm{rec}}(B,\delta)=4\mathfrak R_B(\mathcal G^{\mathrm{rec}})+2M^2\sqrt{\frac{\ln(2/\delta)}{2B}}$. 
\end{theorem}

Theorem~\ref{thm:vanilla_bound} decomposes the lagged prediction error of vanilla RPMs into \textcircled{\small{1}} finite-sample reconstruction and approximation error from sparse optimization histories, and \textcircled{\small{2}} policy-update-induced temporal mismatch from applying risk estimates at $\bm{\theta}_t$ to task selection under $\bm{\theta}_{t+1}$.
Thus, interpolation-based vanilla RPMs~\citep{wang2025modelpredictivetasksampling,gao2026prompt} incur an $\mathcal{O}(\eta)$ lagging bias, which may distort task-risk rankings and weaken task-space generalization.

\subsection{Context-based Extrapolative Risk Prediction}
Having identified temporal misalignment and data scarcity as the main limitations, we introduce ICRPM with two components: (i) \textit{a look-ahead mechanism} that conditions risk prediction on recent optimization context; and (ii) \textit{temporal-weighted trajectory augmentation} with exponential decay weighting to improve efficiency under tight sampling budgets.

\begin{figure}[h]
    \centering
    \includegraphics[width=\textwidth]{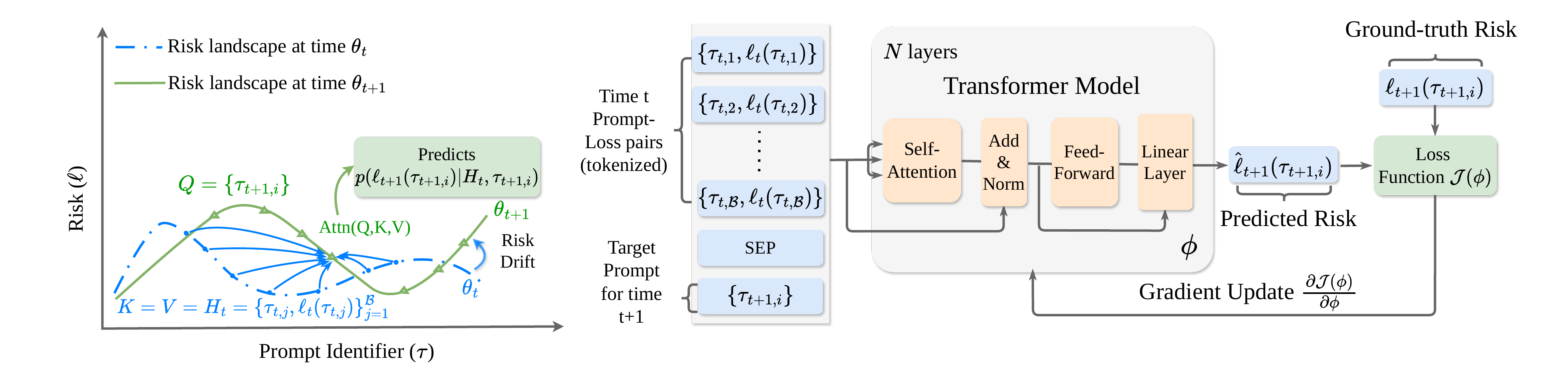}
    \caption{\textbf{Left}: The risk landscape evolves from $\boldsymbol{\theta}_t$ to $\boldsymbol{\theta}_{t+1}$, with historical prompt-loss pairs $H_t$ modeling temporal risk dynamics and predicting future risk using an attention mechanism.
    \textbf{Right}: The Transformer-based ICRPM tokenizes prompt-loss sequences with a target prompt at time $t+1$, and a regression head estimates the future risk $\hat{\ell}_{t+1}$ by minimizing the discrepancy with ground-truth $\ell_{t+1}$ through gradient optimization.}
    \label{fig:ICRPM_pipeline}
\end{figure}

\paragraph{In-Context Look-Ahead Risk Estimation.}
To overcome the single-step temporal misalignment in interpolation-based evaluation, ICRPM adopts an extrapolative formulation
$p_{\phi}(\ell_{t+1}(\boldsymbol{\tau})\mid \boldsymbol{\tau},H_t)$,
where the latest batch $H_t$ serves as the inference-time optimization context.
By directly targeting the future policy state $\boldsymbol{\theta}_{t+1}$, ICRPM predicts next-step task risk rather than reconstructing the lagged risk under $\boldsymbol{\theta}_t$, thereby alleviating the $\mathcal O(\eta)$ temporal bias identified in Theorem~\ref{thm:vanilla_bound}.

\paragraph{Exponentially Decayed Sliding Window Augmentation.}
The sliding window is a data-reuse mechanism for training the look-ahead predictor
under limited task evaluations. ICRPM maintains the $W$ most recent history
batches and trains on consecutive context--target pairs
$\left\{\left\{(\boldsymbol{\tau}_{t-k+1,i},
\ell_{t-k+1}(\boldsymbol{\tau}_{t-k+1,i}),H_{t-k})\right\}_{i=1}^{B}\right\}_{k=1}^{W}$,
weighted by $\gamma^k$. This reuse mitigates the sparse-history issue in
Theorem~\ref{thm:vanilla_bound}: recent batches provide more training pairs,
while exponential decay gives less weight to outdated histories.

\paragraph{Optimization Objective.}
Combining the look-ahead formulation and historical augmentation, ICRPM minimizes a temporal-weighted negative log-likelihood over in-context prediction targets: 
\begin{equation}
\min_{\boldsymbol{\phi}}\mathcal{J}(\boldsymbol{\phi})=-\sum_{k=1}^{W}\gamma^k \left(\frac{1}{B}\sum_{i=1}^{B}\ln p_{\boldsymbol{\phi}}\left(\ell_{t-k+1}(\bm{\tau}_{t-k+1,i}) \vert \boldsymbol{\tau}_{t-k+1,i}, H_{t-k} \right) \right)
\label{eq:icrpm_obj}
\end{equation}
where $\boldsymbol{\phi}$ denotes model parameters, $W$ is the look-back window and $\gamma\in (0,1]$ controls temporal decay. 
The objective in Eq.~(\ref{eq:icrpm_obj}) trains ICRPM on already observed one-step transitions:
at round $t$, it uses consecutive history pairs $H_{s-1}\!\to H_s$ with
$s\le t$ from the buffer. The unobserved batch $H_{t+1}$ is not used when
selecting $\mathcal T_{t+1}^B$.
Only after the selected batch is evaluated, the resulting $H_{t+1}$ is appended to the buffer and used as a target in later RPM
updates. 
At the inference stage, i.e., task evaluation under the coming $\bm{\theta}_{t+1}$, ICRPM takes in the latest batch $H_t$ as the context set $\mathcal{C}_t$ and the candidate identifier $\bm{\tau}\in{\mathcal{T}}_{t+1}^{\hat B}$ to predict risk $\hat{\ell}_{t+1}(\bm{\tau})$.

\begin{algorithm}[H]
    \footnotesize
    \caption{ICRPM Training \& Prompt and Robust Predictive Sampling (PPS and RPS)}
    \label{alg:icrpm}
    \SetInd{0.2em}{0.4em}

    \KwIn{
    Task/Prompt prior $p(\bm{\tau})$; 
    history buffer $\mathcal{S}=\{H_i\}$; 
    $\bm{\phi}$;
    $B, \hat{B}$; 
    $W$; 
    $\gamma$;
    $\alpha$;
    $\texttt{mode}\in\{\texttt{PPS},\texttt{RPS}\}$; 
    target $\rho^*$.
    }
    \KwOut{Selected batch $\mathcal{T}_{t+1}^{B}$}

    \hrule\vspace{0.5em}

    \noindent
    \begin{minipage}[t]{0.47\textwidth}
        \tcp{Online ICRPM Optimization}
        \For{each new batch $H_t$}{
            \tcp{FIFO}
            $\mathcal{S} \leftarrow \mathcal{S} \cup \{H_t\}$\;
            \lIf{$|\mathcal{S}| > W+1$}{Drop oldest batch in $\mathcal{S}$}

            $\mathcal{J}(\bm{\phi}) \leftarrow 0$\;

            \For{$k = 1$ \KwTo $\min(W,|\mathcal{S}|-1)$}{
                \tcp{Context vs. look-ahead}
                $H_{ctx}, H_{tgt} = H_{t-k}, H_{t-k+1}$\;
                $\mathcal{L}_{k} \leftarrow
                -\frac{1}{B}
                \sum_{j}
                \ln p_{\bm{\phi}}
                (
                y_{tgt,j}
                \mid
                \bm{\tau}_{tgt,j}, H_{ctx}
                )$\;
                $\mathcal{J}(\bm{\phi}) \mathrel{+}= \gamma^{k}\mathcal{L}_{k}$\;
            }
            \tcp{Update $\bm{\phi}$ by minimizing $\mathcal{J}(\bm{\phi})$}
             $\bm{\phi}$ = $\bm{\phi}$ - $\alpha\nabla_{\bm{\phi}}\mathcal{J}(\bm{\phi})$\;
        }
    \end{minipage}%
    \hfill
    \vrule %
    \hfill
    \begin{minipage}[t]{0.47\textwidth}
        \tcp{Predictive Task/Prompt Selection}
        Sample candidate pool: \\
        $\mathcal{T}_{t+1}^{\hat{B}} = \{\hat{\bm{\tau}}_{t+1,i}\}_{i=1}^{\hat{B}} \sim p(\bm{\tau})$\;

        \For{$i = 1$ \KwTo $\hat{B}$}{
            \tcp{One-step look-ahead}
            $\hat{y}_{i}^{\,t+1} = E_{p_{\bm{\phi}} ( y \mid \hat{\bm{\tau}}_{t+1, i}, H_t )}[y]$\;
        }

        \eIf{$\texttt{mode}=\texttt{PPS}$}{
            $c_{t+1}^{*}\leftarrow \rho^*$ \tcp*{$\hat{y}$ is success rate}
        }{
            $c_{t+1}^{*}\leftarrow \max_{\bm{\tau}} \hat{y}_{\bm{\tau}}^{\,t+1}$ \tcp*{$\hat{y}$ is risk}
        }

        $\mathcal{T}_{t+1}^{B}
        =
        \operatorname*{arg\,top-}B_{\bm{\tau} \in \mathcal{T}_{t+1}^{\hat{B}}}
        \left(
        -\big\|
        \hat{y}_{\bm{\tau}}^{\,t+1}
        -
        c_{t+1}^{*}
        \big\|_2^2
        \right)$\;

        \Return{$\mathcal{T}_{t+1}^{B}$}
    \end{minipage}
    \vspace{0.5em}
\end{algorithm}

\paragraph{First-order interpretation.}
The look-ahead objective in Eq.~(\ref{eq:icrpm_obj}) aligns RPM supervision with the policy state used for task
selection. The following proposition gives a lightweight formalization: if recent optimization
history enables the in-context predictor to capture the local risk drift, then the next-step prediction
error is controlled by a context approximation term and a second-order smoothness residual.
\begin{proposition}[Look-ahead reduces temporal mismatch]
\label{prop:lookahead-mismatch}
Let $\bm{\Delta}_t=\bm{\theta}_{t+1}-\bm{\theta}_t$ with
$\|\bm{\Delta}_t\|\leq \eta G$, and assume that $\ell(\bm{\theta},\bm{\tau})$ is
$\beta$-smooth in $\bm{\theta}$. For analysis, define the first-order look-ahead risk as
$m_t(\bm{\tau})=\ell_t(\bm{\tau})+
\nabla_{\bm{\theta}}\ell(\bm{\theta}_t,\bm{\tau})^\top\bm{\Delta}_t$.
If the ICRPM predictor satisfies
$\mathbb{E}_{\bm{\tau}\sim p}
|\hat{\ell}^{\mathrm{IC}}_{t+1}(\bm{\tau})-m_t(\bm{\tau})|
\leq\epsilon_{\mathrm{IC}}$, then
\begin{equation}
\mathbb{E}_{\bm{\tau}\sim p}
\left[
|\hat{\ell}^{\mathrm{IC}}_{t+1}(\bm{\tau})-\ell_{t+1}(\bm{\tau})|
\right]
\leq
\epsilon_{\mathrm{IC}}
+
\frac{\beta}{2}\eta^2G^2 .
\end{equation}
\end{proposition}

This result provides a first-order explanation for the advantage of look-ahead risk prediction.
A lagged RPM estimates $\ell_t$ for task selection under $\bm{\theta}_{t+1}$ and therefore retains
the $O(\eta)$ temporal mismatch identified in Theorem~\ref{thm:vanilla_bound}. In contrast, ICRPM directly targets
the next-step risk: when the context predictor captures the local risk drift, the leading lagging bias
is replaced by the in-context approximation error $\epsilon_{\mathrm{IC}}$ and a second-order residual.

\subsection{Implementation and Neural Architecture}
Algorithm~\ref{alg:icrpm} summarizes the implementation; the selected batch is passed to downstream optimizers such as robust Meta-RL, DR, or LLM RLVR.
ICRPM performs the extrapolative formulation $p(\ell \vert \tau, H_{t}; \boldsymbol{\theta}_{t+1})$. 
At each iteration, the context set $\mathcal{C}_{t-1} = \{(\boldsymbol{\tau}_i, \ell_{t-1}(\boldsymbol{\tau}_i))\}_{i=1}^{B}$ consists of task-risk pairs evaluated under $\boldsymbol{\theta}_{t-1}$, while each training target is a pair $(\boldsymbol{\tau}, \ell_{t}(\boldsymbol{\tau}))$ under the last policy state $\boldsymbol{\theta}_{t}$. 
This yields $B$ supervised training episodes per iteration, one for each $j \in [B]$, without any additional environment interaction.
\paragraph{Risk Sequence Tokenization and Cross-Attention.}
ICRPM has a simple design. As illustrated in Fig.~\ref{fig:ICRPM_pipeline}, we tokenize the task-risk pairs and adopt the Transformer as the default neural architecture to capture the correlation between the target and past risk sequences.
This way encodes the context into a permutation-invariant representation and produces $p(\ell \vert \tau, H_{t}; \boldsymbol{\theta}_{t+1})$ via cross-attention.

\section{Results and Analysis}
\label{sec:results}

This section empirically examines ICRPM's influence within MPTS. Motivated by the failure modes diagnosed in Section~\ref{sec:theory_vanilla}, the experiments are designed to answer two Research Questions (RQs):
\begin{enumerate}
    \item Do extrapolative RPMs enhance task evaluation reliability over iteration?
    \item How does task evaluation reliability influence performance in downstream applications?
\end{enumerate}
We track Spearman Rank Correlation (SRC) between predicted task difficulty and ground-truth task risk as a diagnostic measure of rank reliability, and report CVaR$_\alpha$~\citep{rockafellar2000optimization} for robust optimization scenarios and test average performance for RLVR. 
Full experimental details are provided in Appendix~\ref{expdetails}.

\paragraph{Benchmarks and Baselines.}
We evaluate ICRPM on few-shot sinusoid regression, Meta-RL, and DR, and LLM post-training via RLVR. 
(i) For regression and continuous-control, we retain the standard MPTS protocol~\citep{wang2025modelpredictivetasksampling}, e.g., the same task distribution, batch size, candidate pool, selection budget, and downstream optimizer; and baselines include ERM~\citep{vapnik1998statistical}, GDRM~\citep{sagawa2019distributionally}, DRM~\citep{wang2024simple,greenberg2024train}, OHTM~\citep{kumar2023effect}, MPTS~\citep{wang2025modelpredictivetasksampling}.
(ii) For LLM RL post-training, we adopt the curriculum-based setup of~\citep{gao2026prompt}, and baselines are uniform sampling and SoTA predictive sampling methods such as PCL~\citep{gao2026prompt}, GRESO~\citep{zheng2025act}, and MoPPS~\citep{qu2026can}. 
For both cases, ICRPM replaces only the interpolation-based RPM with a parameter-matched context-conditioned one-step extrapolative predictor.

\subsection{Robust Few-Shot Sinusoid Regression}
\label{sec:sinusoid}

We first validate ICRPM on the classic few-shot sinusoid regression benchmark~\citep{finn2017model}, which provides a controlled setting to isolate the effect of improved task risk estimation.
Tasks are parameterised by amplitude and phase, and MAML is used as the meta-learning backbone.

\begin{figure}[h]
    \centering
    \includegraphics[width=\linewidth]{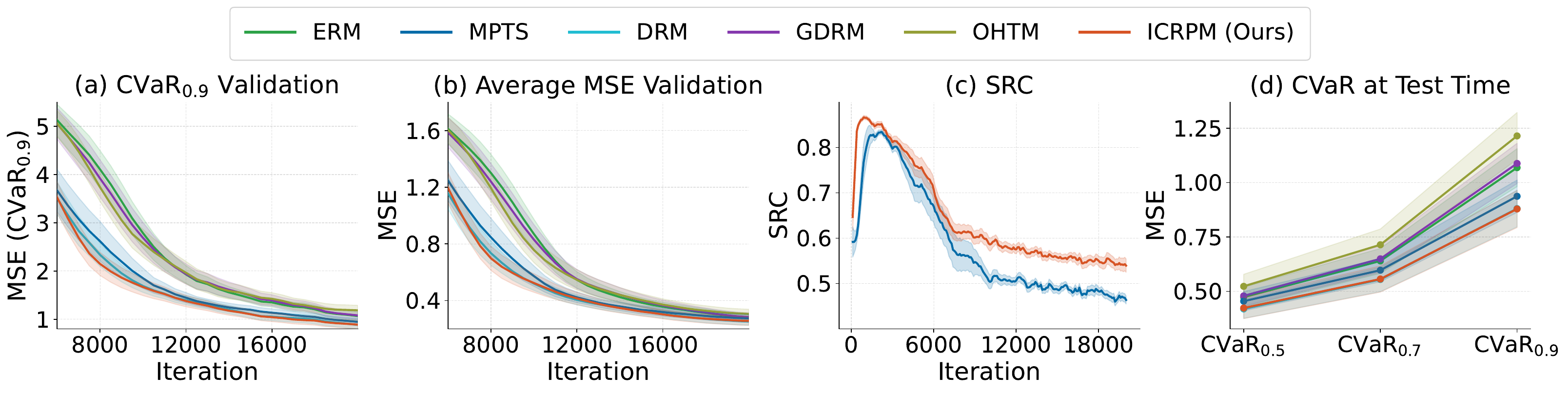}
    \caption{
        \textbf{Few-shot sinusoid regression results (10 Runs, Mean~$\pm$~Standard Error of the Mean (SEM)).}
        From left to right:
        (a)~validation CVaR$_{0.9}$ MSE curves during meta-training;
        (b)~validation average MSE curves during meta-training;
        (c)~SRC of predicted vs.\ ground-truth difficulty during training;
        (d)~meta-testing CVaR$_\alpha$ MSE at multiple robustness levels.
        ICRPM (ours) achieves consistently lower tail-risk error and maintains higher SRC throughout training.
    }
    \label{fig:sinusoid}
\end{figure}

\paragraph{Few-Shot Regression Analysis.}
In Fig.~\ref{fig:sinusoid}, ICRPM consistently outperforms all baselines. It achieves lower average MSE (Fig.~\ref{fig:sinusoid}b) and CVaR$_{0.9}$ MSE (Fig.~\ref{fig:sinusoid}a) with faster convergence, while maintaining higher SRC with ground-truth task risk throughout training (Fig.~\ref{fig:sinusoid}c). At meta-test time (Fig.~\ref{fig:sinusoid}d), ICRPM achieves the lowest CVaR$_\alpha$ MSE across all $\alpha\in\{0.5,0.7,0.9\}$, suggesting that more reliable task prioritisation is accompanied by stronger tail-task generalization. Vanilla MPTS underperforms in both convergence stability and worst-case performance.

\subsection{Robust Meta-RL and Domain Randomization}
\label{sec:metarl_dr}

We evaluate ICRPM on Meta-RL over three MuJoCo environments~\citep{finn2017model,todorov2012mujoco} and zero-shot DR on Ergo-Reacher and Lunar-Lander~\citep{mehta2020active}.
Both settings require generalisation across a distribution of MDPs parameterised by continuous task identifiers, where robustness to tail or out-of-distribution tasks is critical.

\begin{figure*}[t]
    \centering
    \includegraphics[width=\textwidth]{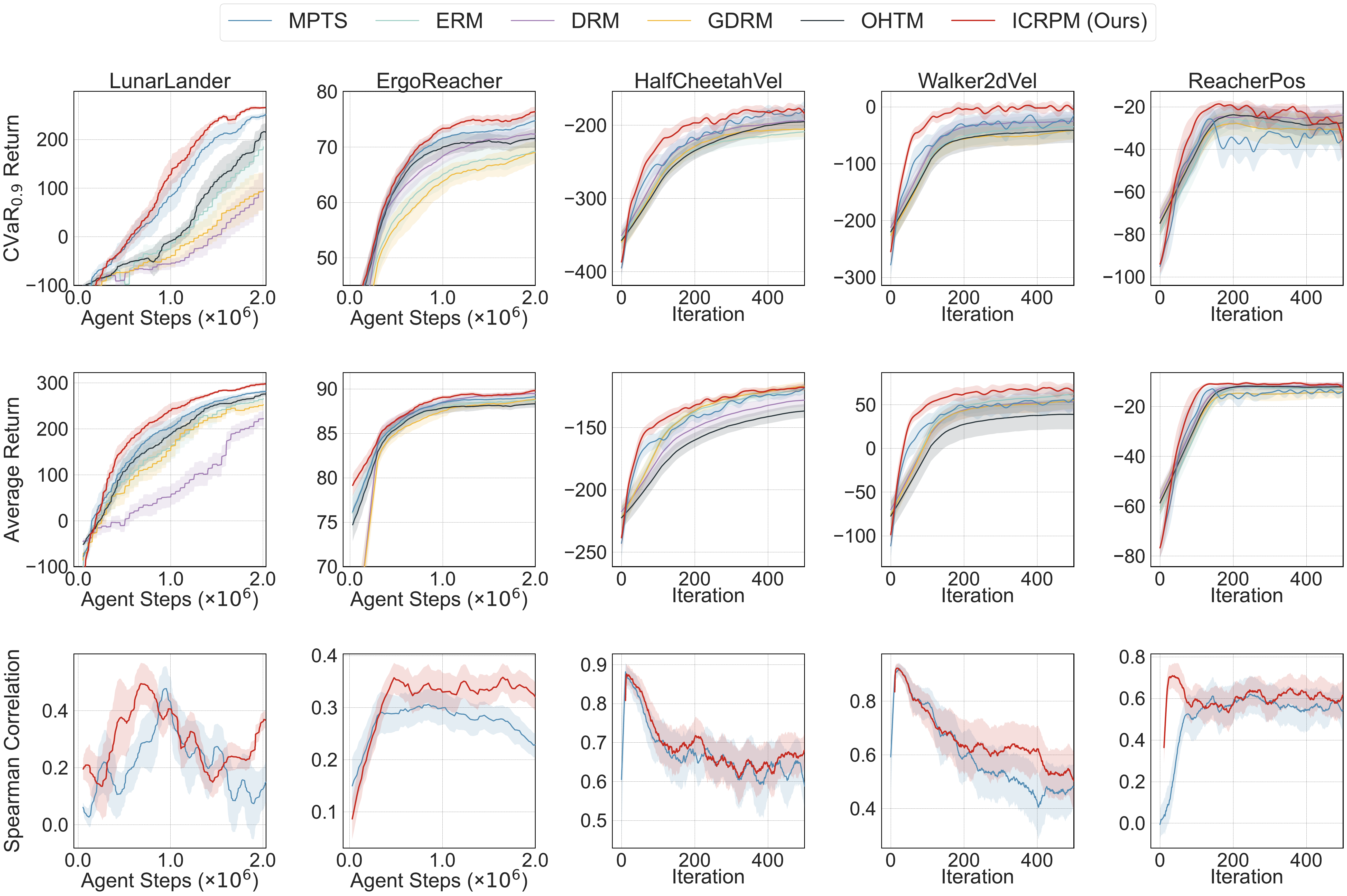}
    \caption{
        \textbf{DR (columns 1--2) and Meta-RL (columns 3--5) learning curves and correlation trends (10 Runs, Mean~$\pm$~SEM).}
        Rows from top to bottom: CVaR$_{0.9}$ returns, average returns, and SRC of predicted vs.\ ground-truth task risk.
        ICRPM consistently attains the highest tail-risk performance while maintaining competitive average returns, alongside superior rank correlation throughout training.
    }
    \label{fig:metarl_dr}
\end{figure*}

\begin{table*}[htbp]
\centering
\caption{\textbf{Evaluation return on five robotic control tasks. (10 Runs, Mean~$\pm$~SEM)} The best results are bold, with the runner-up underlined.}
\label{tab:rl_cvar_5env_results}
\vspace{2pt}
\renewcommand{\arraystretch}{1.3}
\setlength{\tabcolsep}{3.5pt}

\resizebox{\textwidth}{!}{%
\begin{tabular}{l|cc|cc|cc|cc|cc}
\toprule

\textbf{Method}
& \multicolumn{2}{c|}{\textbf{HalfCheetahVel}}
& \multicolumn{2}{c|}{\textbf{Walker2dVel}}
& \multicolumn{2}{c|}{\textbf{ReacherPos}}
& \multicolumn{2}{c|}{\textbf{LunarLander}}
& \multicolumn{2}{c}{\textbf{ErgoReacher}} \\

\cmidrule(lr){2-3}
\cmidrule(lr){4-5}
\cmidrule(lr){6-7}
\cmidrule(lr){8-9}
\cmidrule(lr){10-11}

& CVaR$_{0.5}$ & CVaR$_{0.9}$
& CVaR$_{0.5}$ & CVaR$_{0.9}$
& CVaR$_{0.5}$ & CVaR$_{0.9}$
& CVaR$_{0.5}$ & CVaR$_{0.9}$
& CVaR$_{0.5}$ & CVaR$_{0.9}$ \\
\midrule

ERM~\citep{vapnik1998statistical}
& -153.6{\scriptsize{$\pm$3.9}} & -209.9{\scriptsize{$\pm$7.6}}
& 23.5{\scriptsize{$\pm$15.6}} & -34.2{\scriptsize{$\pm$20.7}}
& -17.5{\scriptsize{$\pm$1.9}} & -31.1{\scriptsize{$\pm$5.4}}
& 243.0{\scriptsize{$\pm$14.2}} & 185.5{\scriptsize{$\pm$30.4}}
& 83.8{\scriptsize{$\pm$0.7}} & 70.1{\scriptsize{$\pm$1.4}} \\

DRM~\citep{wang2024simple,greenberg2024train}
& -151.1{\scriptsize{$\pm$6.2}} & -195.5{\scriptsize{$\pm$8.4}}
& \underline{26.1}{\scriptsize{$\pm$12.2}} & \underline{-26.0}{\scriptsize{$\pm$14.7}}
& \underline{-16.0}{\scriptsize{$\pm$2.2}} & \textbf{-26.5}{\scriptsize{$\pm$5.6}}
& 184.2{\scriptsize{$\pm$25.0}} & 96.8{\scriptsize{$\pm$36.6}}
& \underline{85.3}{\scriptsize{$\pm$0.4}} & 72.9{\scriptsize{$\pm$1.0}} \\

GDRM~\citep{sagawa2019distributionally}
& -149.8{\scriptsize{$\pm$4.3}} & -206.4{\scriptsize{$\pm$8.0}}
& 16.1{\scriptsize{$\pm$13.3}} & -42.8{\scriptsize{$\pm$16.6}}
& -19.6{\scriptsize{$\pm$4.3}} & -30.7{\scriptsize{$\pm$6.1}}
& 207.9{\scriptsize{$\pm$18.4}} & 93.7{\scriptsize{$\pm$38.0}}
& 84.4{\scriptsize{$\pm$0.7}} & 70.2{\scriptsize{$\pm$1.7}} \\

OHTM~\citep{kumar2023effect}
& -160.0{\scriptsize{$\pm$7.3}} & -194.8{\scriptsize{$\pm$10.3}}
& 7.6{\scriptsize{$\pm$21.5}} & -40.8{\scriptsize{$\pm$21.3}}
& -16.3{\scriptsize{$\pm$1.4}} & -28.9{\scriptsize{$\pm$5.6}}
& 258.7{\scriptsize{$\pm$6.4}} & 215.7{\scriptsize{$\pm$17.9}}
& 83.8{\scriptsize{$\pm$0.7}} & 72.6{\scriptsize{$\pm$2.0}} \\

MPTS~\citep{wang2025modelpredictivetasksampling}
& \underline{-143.6}{\scriptsize{$\pm$4.6}} & \underline{-184.0}{\scriptsize{$\pm$8.3}}
& 24.6{\scriptsize{$\pm$16.2}} & -27.9{\scriptsize{$\pm$19.1}}
& -20.1{\scriptsize{$\pm$3.4}} & -33.2{\scriptsize{$\pm$8.0}}
& \underline{268.5}{\scriptsize{$\pm$3.4}} & \underline{251.8}{\scriptsize{$\pm$4.1}}
& 84.7{\scriptsize{$\pm$0.5}} & \underline{75.7}{\scriptsize{$\pm$1.1}} \\

\midrule
\rowcolor{myhighlight}
\textbf{ICRPM (Ours)}
& \textbf{-140.5}{\scriptsize{$\pm$3.5}} & \textbf{-180.2}{\scriptsize{$\pm$6.4}}
& \textbf{43.2}{\scriptsize{$\pm$6.7}} & \textbf{-1.5}{\scriptsize{$\pm$7.2}}
& \textbf{-15.2}{\scriptsize{$\pm$0.9}} & \underline{-27.8}{\scriptsize{$\pm$3.8}}
& \textbf{285.3}{\scriptsize{$\pm$2.6}} & \textbf{265.8}{\scriptsize{$\pm$2.1}}
& \textbf{86.5}{\scriptsize{$\pm$0.4}} & \textbf{78.3}{\scriptsize{$\pm$0.7}} \\

\bottomrule
\end{tabular}%
}
\end{table*}

\paragraph{Meta-RL Analysis.}
Fig.~\ref{fig:metarl_dr} (columns 3--5) reports training dynamics on HalfCheetahVel, Walker2dVel, and ReacherPos. ICRPM achieves the highest CVaR$_{0.9}$ validation returns on all three benchmarks while preserving competitive average performance --- a balance DRM fails to maintain, notably degrading on Walker2dVel. The SRC curves (row~3) confirm that the Transformer-based RPM with temporal-weighted augmentation sustains markedly higher rank correlation than MPTS throughout training. At meta-test time (Table~\ref{tab:rl_cvar_5env_results}), ICRPM attains the best CVaR$_\alpha$ returns across nearly all environments and robustness levels, with especially large gains on Walker2dVel (CVaR$_{0.9}$: $-1.5$ vs.\ $-27.9$ for MPTS) and substantially reduced variance. 

\begin{wrapfigure}[10]{r}{0.38\textwidth}
    \centering
    \vspace{-12pt}
    \includegraphics[width=0.34\textwidth]{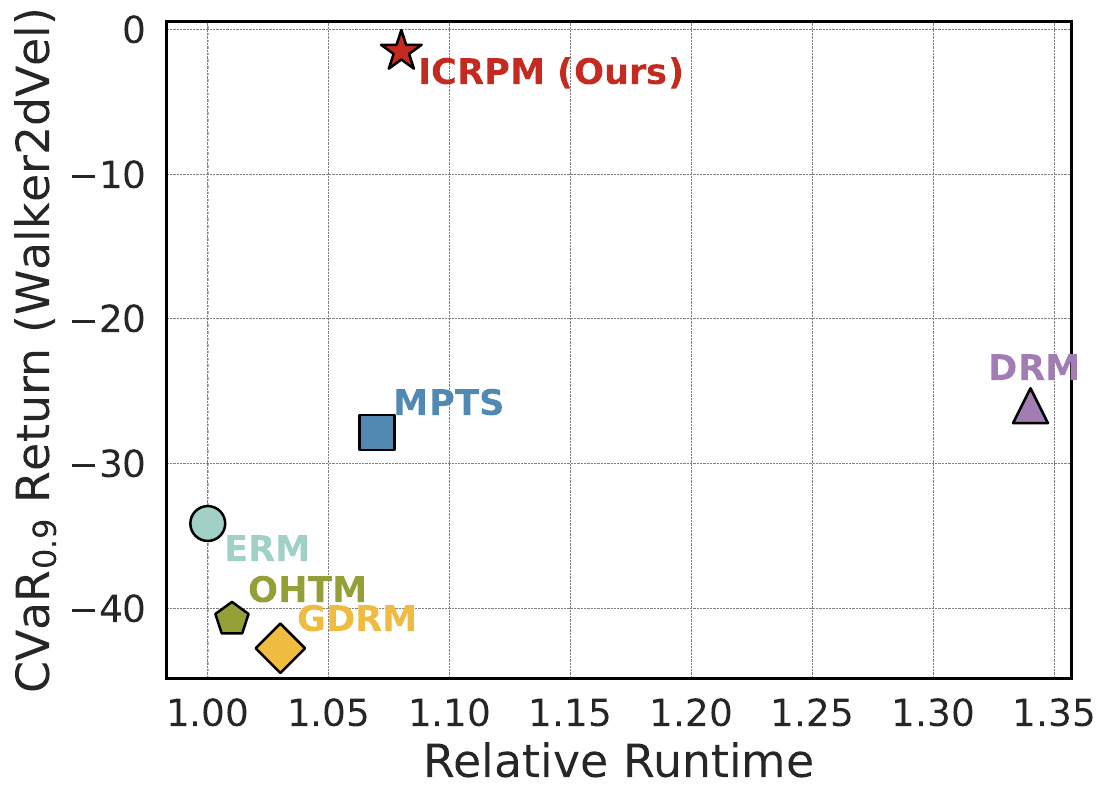}
    \caption{\textbf{CVaR$_{0.9}$ return vs.\ relative runtime on Walker2dVel.}}
    \label{fig:runtime}
    \vspace{-10pt}
\end{wrapfigure}

\paragraph{Domain Randomization Analysis.}
On Ergo-Reacher and Lunar-Lander (Fig.~\ref{fig:metarl_dr}, columns 1--2), ICRPM consistently outperforms all baselines in CVaR$_{0.9}$ returns while remaining competitive in average performance.
Notably, on Lunar-Lander ICRPM improves CVaR$_{0.9}$ at convergence by over 40\% relative to ERM, effectively handling the long-tail of adversarial engine-strength configurations.
The SRC advantage of ICRPM over MPTS persists in both DR environments, demonstrating that mitigating temporal lagging bias extends beyond episodic few-shot learning to the zero-shot continuous-control regime.

\paragraph{Computational Overhead.}
Fig.~\ref{fig:runtime} plots CVaR$_{0.9}$ return against relative runtime on Walker2dVel (normalised to ERM).
ICRPM incurs only 8\% overhead (1.08$\times$), comparable to MPTS (1.07$\times$) and far below DRM (1.34$\times$), whose per-iteration importance-weight computation dominates cost.
GDRM and OHTM add minimal time (1.03$\times$ and 1.01$\times$) but lag considerably in tail-risk performance.
The scatter clearly shows that ICRPM occupies the ideal upper-left region with the highest CVaR$_{0.9}$ return at near-baseline computational cost, confirming that its performance gains come at negligible additional runtime budget.

\subsection{Sample-Efficient RLVR on Countdown}
\label{sec:rlvr}

To evaluate the effectiveness of ICRPM in accelerating LLM post-training, we conduct experiments on the Countdown symbolic reasoning task.
Countdown requires composing given numbers with arithmetic operations to reach a target value, offering verifiable binary rewards that enable straightforward RLVR training.

\paragraph{Experimental Setup.}
We use \textbf{Qwen3-4B-Base}~\citep{yangqwen3} as the policy backbone.
Following the RLVR paradigm, we employ Group Relative Policy Optimization (GRPO)~\citep{shao2024deepseekmath} to optimize the model with verifiable rewards that consider both mathematical correctness and format adherence.
Training is performed on a 20k-instance subset of Countdown-34 (CD34); evaluation is conducted on both CD34 and the harder Countdown-4 (CD4) variant.

\begin{figure}[htbp]
    \centering
    \begin{minipage}{0.62\textwidth}
        \centering
        \includegraphics[width=\linewidth]{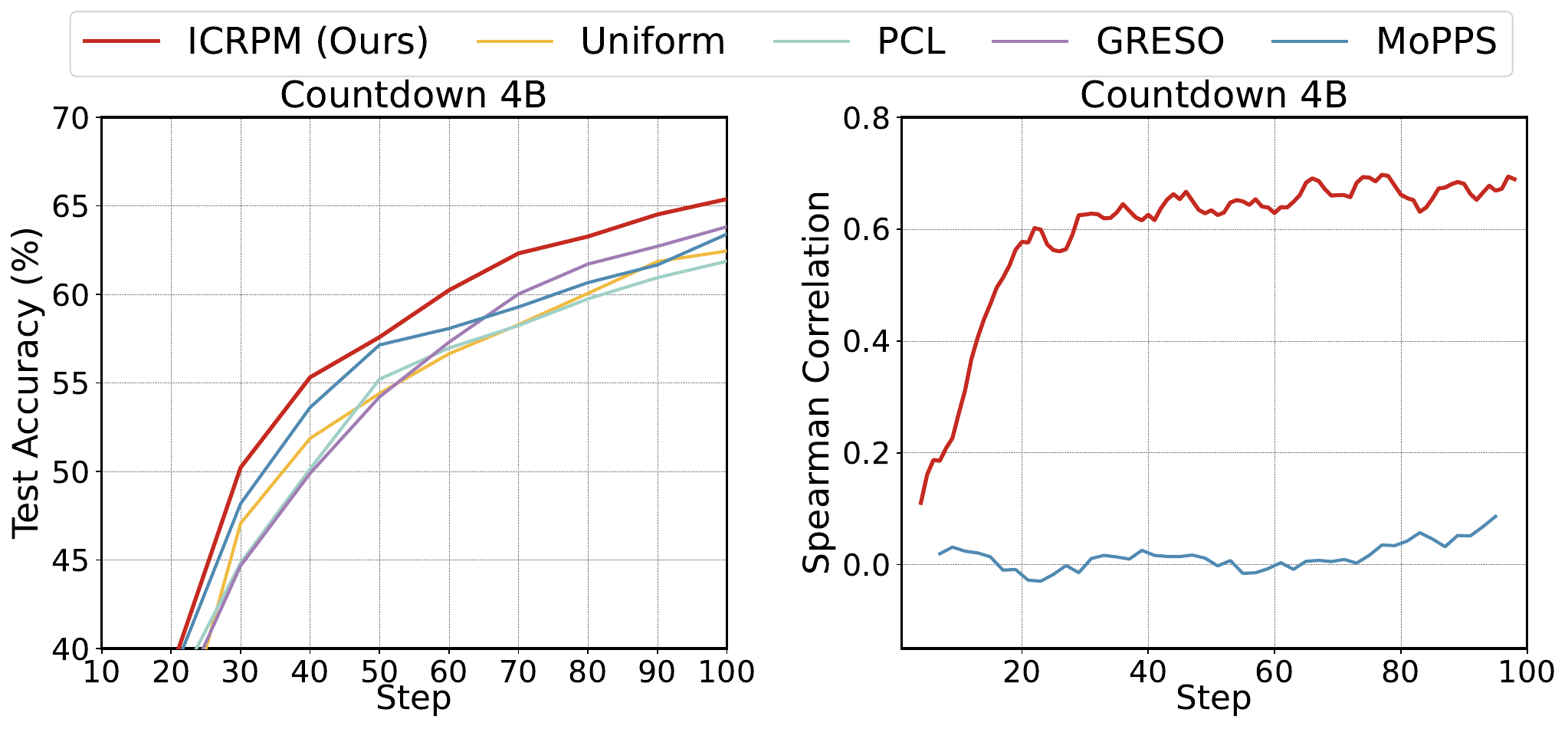}
        \caption{
            \textbf{Test accuracy (left) and SRC (right) on Countdown 4B.}
            ICRPM demonstrates significantly faster convergence and higher correlation with ground-truth difficulty throughout RL post-training.
        }
        \label{fig:cd4_curve}
    \end{minipage}
    \hfill
    \begin{minipage}{0.35\textwidth}
    \raggedright %
    \captionof{table}{
        \textbf{Evaluation results on Countdown.}
        `+' denotes finetuning with the corresponding method.
        \textbf{Bold}: best; \underline{underline}: runner-up.
    }
    \label{tab:cdeval}
    \resizebox{\linewidth}{!}{%
        \begin{tabular}{lccc} %
            \toprule
            \textbf{Method} & \textbf{CD4} & \textbf{CD34} & \textbf{Avg.} \\
            \midrule
            Qwen3-4B~\citep{yangqwen3}        & 1.3  & 3.5  & 2.4  \\
            +Uniform        & 51.1 & 73.8 & 62.5 \\
            +PCL~\citep{gao2026prompt}           & 51.0 & 72.8 & 61.9 \\
            +GRESO~\citep{zheng2025act}          & \underline{53.8} & 73.8 & \underline{63.8} \\
            +MoPPS~\citep{qu2026can}          & 52.9 & \underline{73.9} & 63.4 \\
            \midrule
            \rowcolor{myhighlight}
            +\textbf{ICRPM} & \textbf{55.7} & \textbf{75.1} & \textbf{65.4} \\
            \bottomrule
        \end{tabular}
    }
\end{minipage}
\end{figure}

\paragraph{Curriculum RLVR Analysis.}
As summarised in Table~\ref{tab:cdeval} and Fig.~\ref{fig:cd4_curve}, ICRPM consistently outperforms all baselines. On CD4, ICRPM achieves \textbf{55.7\%} accuracy, surpassing GRESO by 1.9 points and MoPPS by 2.8 points, with similar gains on CD34. 
The training curve (Fig.~\ref{fig:cd4_curve}, left) shows faster convergence and superior sample efficiency, while the SRC curve (right) indicates more reliable online prompt-difficulty rankings as the policy evolves. This addresses the sparse-history and temporal-lag issues discussed in Section~\ref{sec:theory_vanilla}: unlike prompt-specific models such as MoPPS, ICRPM shares information across prompts and better tracks non-stationary difficulty during RL post-training.

\subsection{Ablation Studies}
\label{sec:ablation}

\begin{figure}[h]
    \centering

    \begin{minipage}[t]{0.48\linewidth}
        \centering
        \includegraphics[width=\linewidth]{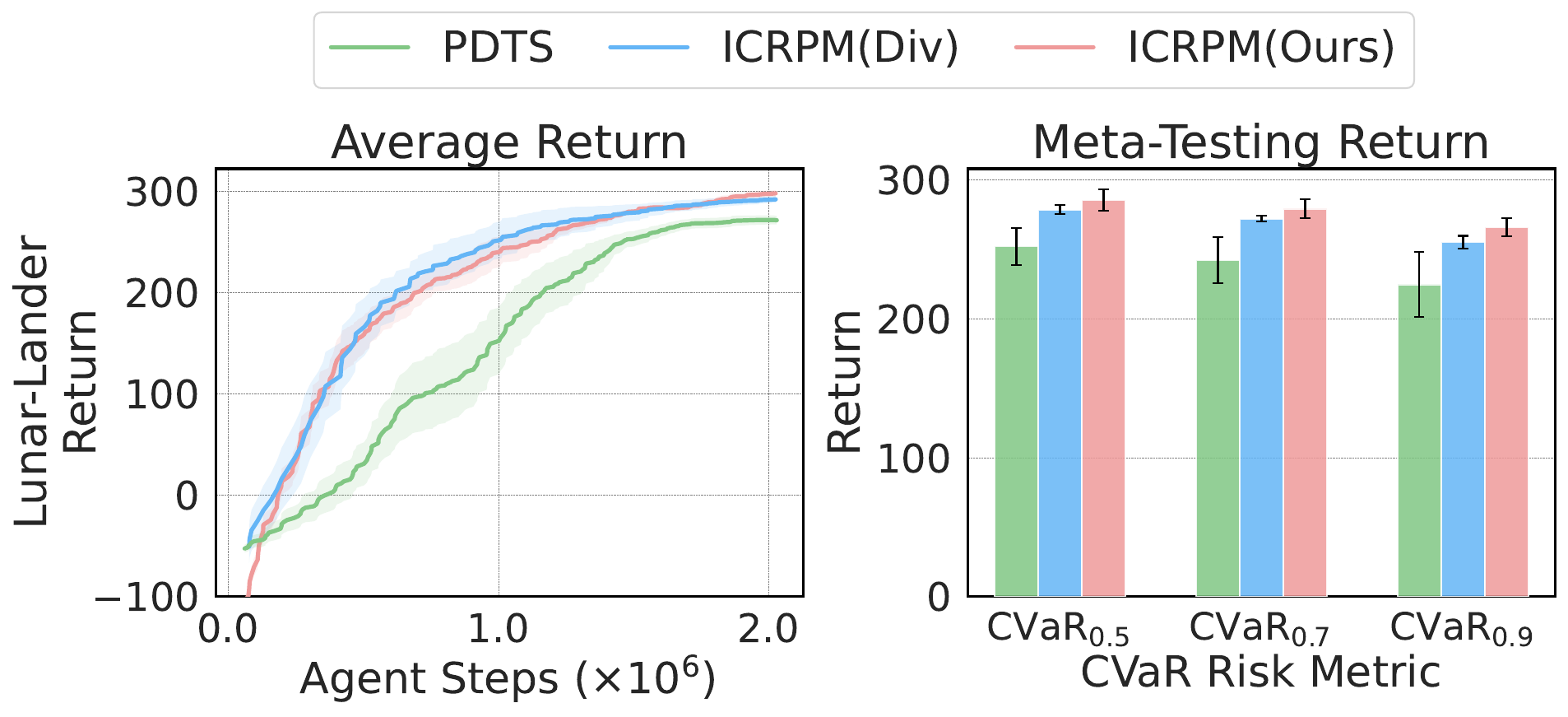}
    \end{minipage}
    \hfill
    \begin{minipage}[t]{0.48\linewidth}
        \centering
        \includegraphics[width=\linewidth]{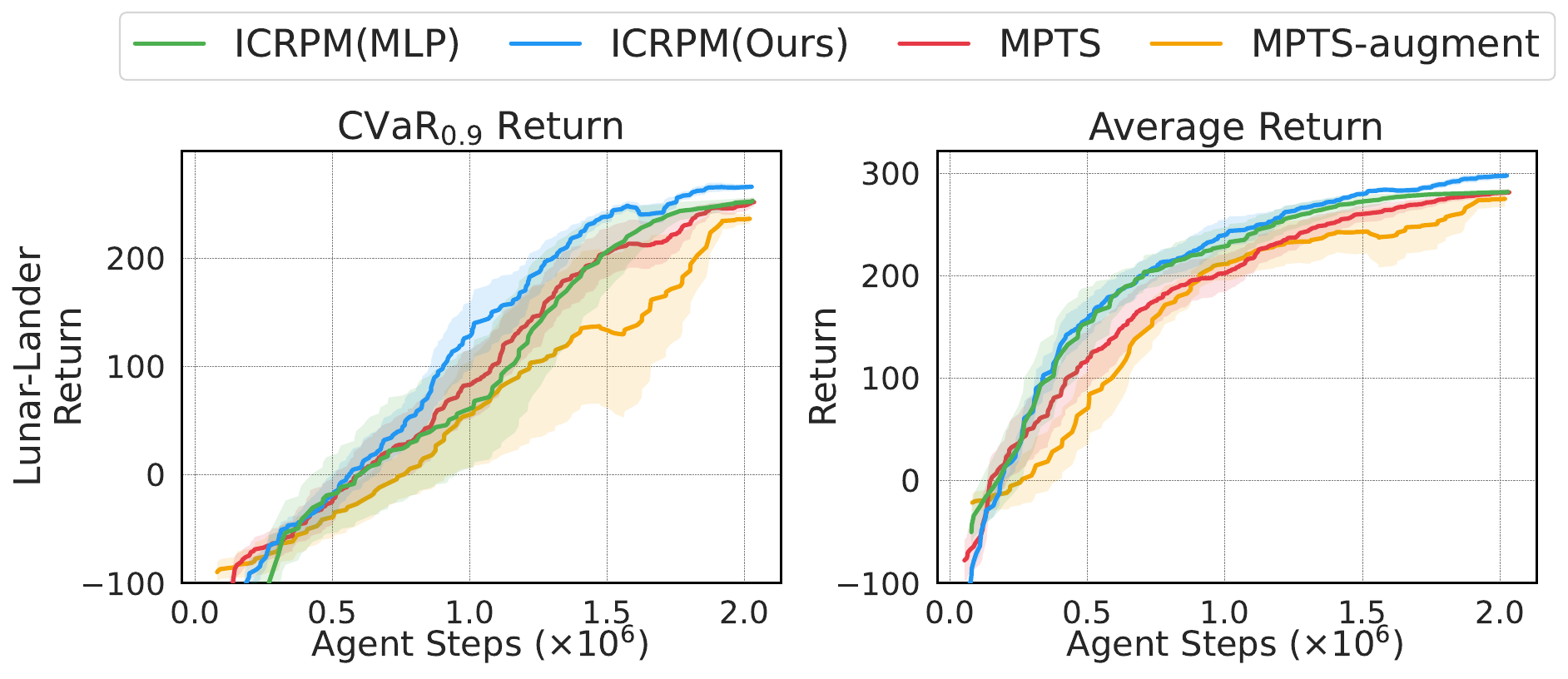}
    \end{minipage}

    \caption{
        \textbf{Ablation studies on Lunar-Lander (10 Runs).}
        \emph{Left}: Impact of acquisition principles.
        Two ICRPM variants (worst-case and diversity-aware) versus PDTS~\citep{qu2025fast}.
        \emph{Right}: Architecture ablation comparing ICRPM with MLP or Transformer backbone, vanilla MPTS, and MPTS with history augmentation.
    }
    \label{fig:ablation}
\end{figure}

We ablate the subset acquisition principles, RPM backbones, window size $W$, and decay factor $\gamma$ on Lunar-Lander. 
Fig.~\ref{fig:ablation} shows that ICRPM improves over PDTS under both worst-case and diversity-aware acquisition. 
For architecture and history augmentation, ICRPM with a parameter-matched MLP still outperforms vanilla MPTS, while MPTS augmented with the same three-step window and $\gamma=0.9$ remains below ICRPM and ICRPM (MLP). 
This indicates that the gains mainly come from context-conditioned look-ahead risk prediction, rather than a specific acquisition rule, Transformer backbone, or history-reuse heuristic.
Together, these ablations match the decomposition in Theorem~\ref{thm:vanilla_bound}: look-ahead prediction addresses temporal lag, while windowed reuse mitigates sparse history. 
Fig.~\ref{fig:wag} shows that $W$ has a stronger effect than $\gamma$: we choose $W=3$ because it gives the best CVaR$_{0.9}$ performance in the ablation, and use $\gamma=0.9$ across all benchmarks as a robust default that mildly discounts older histories.

\begin{figure}[h]
    \centering

    \begin{minipage}[t]{0.48\linewidth}
        \centering
        \includegraphics[width=\linewidth]{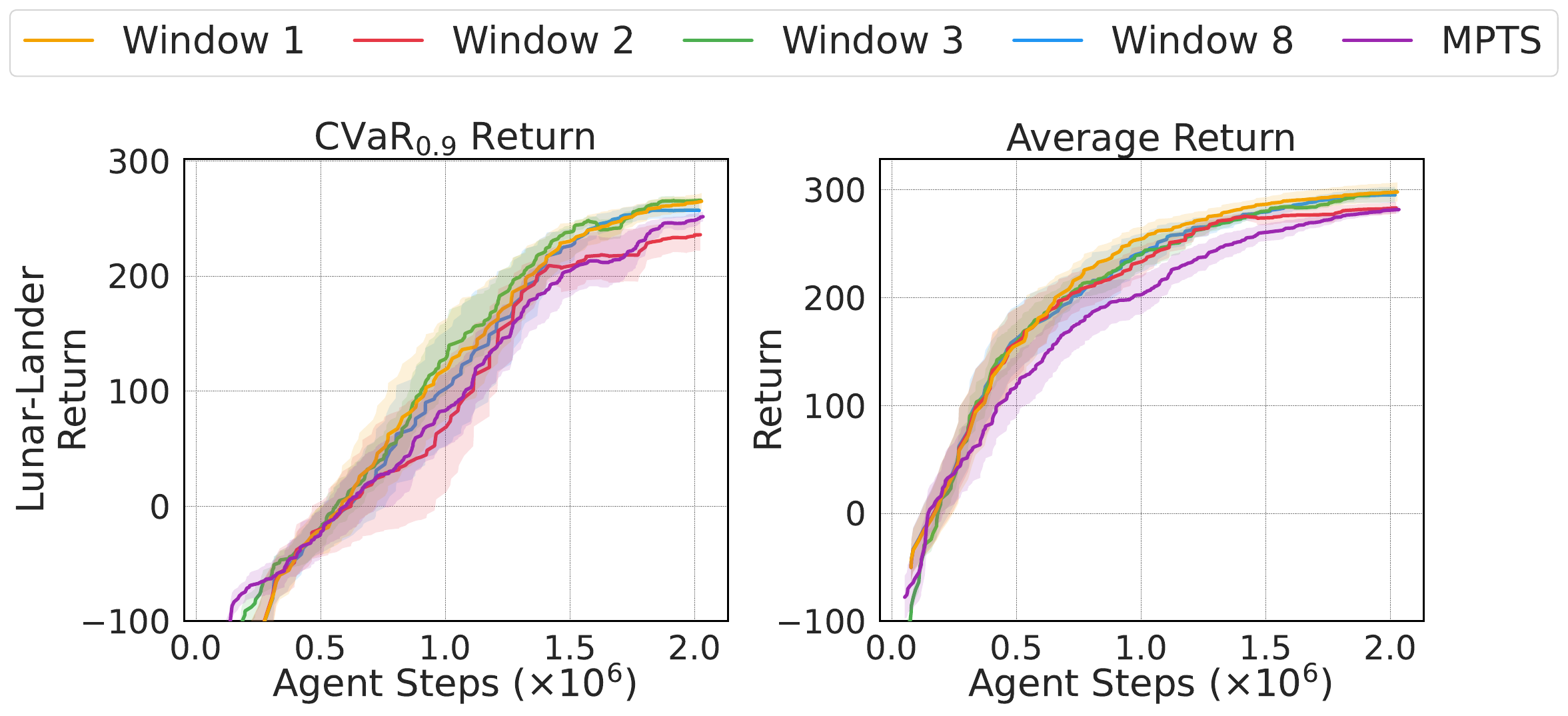}
    \end{minipage}
    \hfill
    \begin{minipage}[t]{0.48\linewidth}
        \centering
        \includegraphics[width=\linewidth]{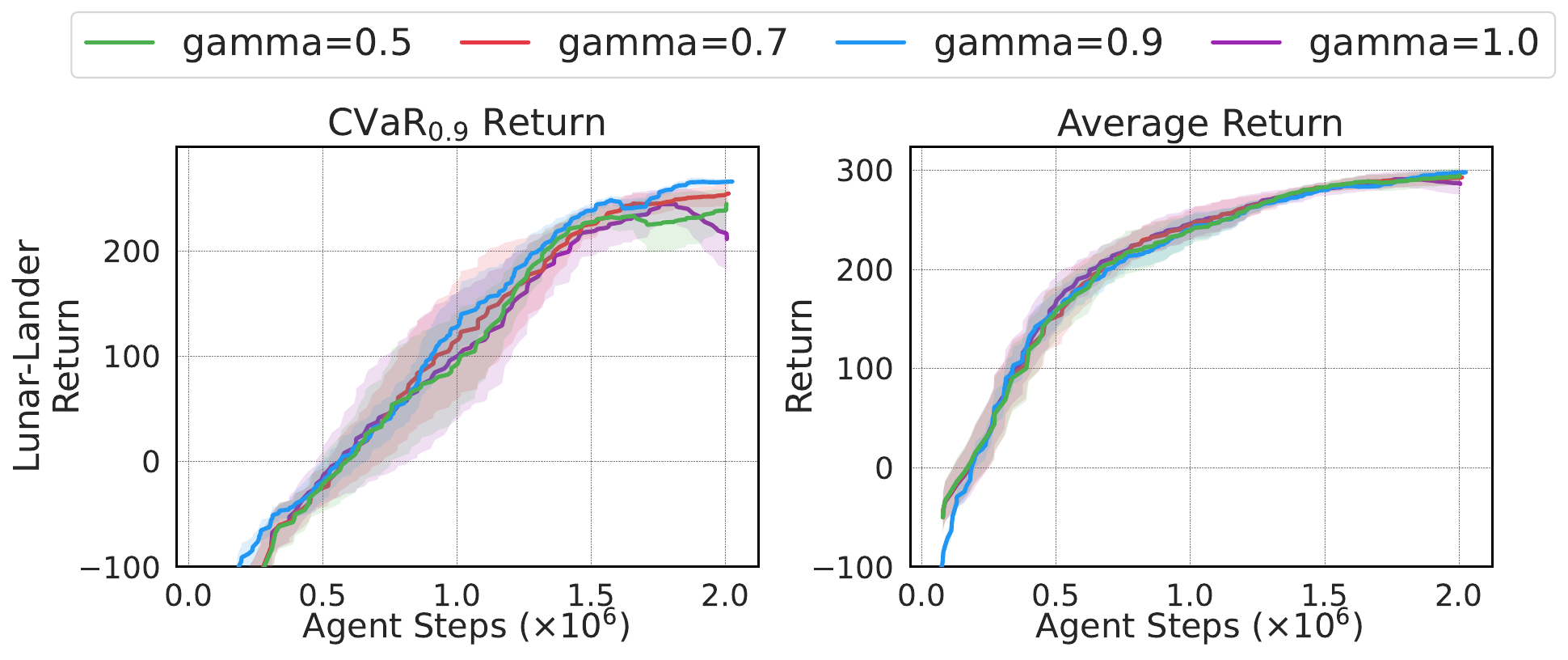}
    \end{minipage}

    \caption{
        \textbf{Sensitivity to $W$ and $\gamma$ on Lunar-Lander (10 Runs).}
        \emph{Left}: Sensitivity to look-back window $W$.
        CVaR$_{0.9}$ returns, average returns under $W \in \{1,2,3,8\}$.
        \emph{Right}: Sensitivity to decay factor $\gamma$.
        CVaR$_{0.9}$ returns, average returns under $\gamma \in \{0.5,0.7,0.9,1.0\}$. 
    }
    \label{fig:wag}
\end{figure}

\section{Conclusion}

This work identifies a structural limitation of MPTS: vanilla RPMs estimate risks under the current policy but use them for post-update task selection, inducing first-order lagging bias and unstable task rankings. We propose ICRPM, which reframes risk prediction as context-conditioned one-step extrapolation with temporal-weighted history augmentation. Empirically, ICRPM improves risk ranking, sample efficiency, and tail-risk performance across regression, robust Meta-RL, domain randomization, and RLVR-based LLM post-training, suggesting that task samplers should model policy-induced risk dynamics rather than rely on static or lagged difficulty estimates.

\bibliographystyle{unsrtnat}
\bibliography{refs}

@inproceedings{finn2017model,
  title={Model-agnostic meta-learning for fast adaptation of deep networks},
  author={Finn, Chelsea and Abbeel, Pieter and Levine, Sergey},
  booktitle={International conference on machine learning},
  pages={1126--1135},
  year={2017},
  organization={PMLR}
}

@inproceedings{taolatent2026,
  title={Latent Space Robust Optimization of Neural Processes with Aligned Stratified Order-Statistic Loss Reduction},
  author={Tao, Qi and Wen, Jiarong and Yang, Jing and Wu, Guanlin and Kaiyu, Zhang and Lv, Yiqin and Du, Wumei and Liang, Xingxing and Wang, Qi},
  year={2026},
  booktitle={Forty-third International Conference on Machine Learning}
}

@article{qu2026listwise,
  title={Listwise Policy Optimization: Group-based RLVR as Target-Projection on the LLM Response Simplex},
  author={Qu, Yun and Wang, Qi and Mao, Yixiu and Zou, Heming and Jiang, Yuhang and Li, Yingyue and Xu, Wutong and Cai, Lizhou and Liu, Weijie and Bai, Clive and others},
  journal={arXiv preprint arXiv:2605.06139},
  year={2026}
}

@article{mao2026rlvr,
  title={RLVR without Ineffective Samples: Group Prioritized Off-Policy Optimization for LLM Reasoning},
  author={Mao, Yixiu and Qu, Yun and Wang, Qi and Zou, Heming and Ji, Xiangyang},
  journal={arXiv preprint arXiv:2606.01281},
  year={2026}
}

@inproceedings{
gao2026prompt,
title={Prompt Curriculum Learning for Efficient {LLM} Post-Training},
author={Zhaolin Gao and Joongwon Kim and Wen Sun and Thorsten Joachims and Sid Wang and Richard Yuanzhe Pang and Liang Tan},
booktitle={The Fourteenth International Conference on Learning Representations},
year={2026},
url={https://openreview.net/forum?id=zqOCacBD3P}
}

@article{yu2025dapo,
  title={Dapo: An open-source llm reinforcement learning system at scale},
  author={Yu, Qiying and Zhang, Zheng and Zhu, Ruofei and Yuan, Yufeng and Zuo, Xiaochen and Yue, Yu and Dai, Weinan and Fan, Tiantian and Liu, Gaohong and Liu, Lingjun and others},
  journal={arXiv preprint arXiv:2503.14476},
  year={2025}
}

@inproceedings{haarnoja2018soft,
  title={Soft actor-critic: Off-policy maximum entropy deep reinforcement learning with a stochastic actor},
  author={Haarnoja, Tuomas and Zhou, Aurick and Abbeel, Pieter and Levine, Sergey},
  booktitle={International conference on machine learning},
  pages={1861--1870},
  year={2018},
  organization={Pmlr}
}

@inproceedings{rezende2014stochastic,
  title={Stochastic backpropagation and approximate inference in deep generative models},
  author={Rezende, Danilo Jimenez and Mohamed, Shakir and Wierstra, Daan},
  booktitle={International conference on machine learning},
  pages={1278--1286},
  year={2014},
  organization={PMLR}
}

@inproceedings{todorov2012mujoco,
  title={Mujoco: A physics engine for model-based control},
  author={Todorov, Emanuel and Erez, Tom and Tassa, Yuval},
  booktitle={2012 IEEE/RSJ international conference on intelligent robots and systems},
  pages={5026--5033},
  year={2012},
  organization={IEEE}
}

@inproceedings{mehta2020active,
  title={Active domain randomization},
  author={Mehta, Bhairav and Diaz, Manfred and Golemo, Florian and Pal, Christopher J and Paull, Liam},
  booktitle={Conference on Robot Learning},
  pages={1162--1176},
  year={2020},
  organization={PMLR}
}

@inproceedings{kumar2023effect,
  title={The effect of diversity in meta-learning},
  author={Kumar, Ramnath and Deleu, Tristan and Bengio, Yoshua},
  booktitle={Proceedings of the AAAI Conference on Artificial Intelligence},
  volume={37},
  number={7},
  pages={8396--8404},
  year={2023}
}

@article{shao2024deepseekmath,
  title={Deepseekmath: Pushing the limits of mathematical reasoning in open language models},
  author={Shao, Zhihong and Wang, Peiyi and Zhu, Qihao and Xu, Runxin and Song, Junxiao and Bi, Xiao and Zhang, Haowei and Zhang, Mingchuan and Li, YK and Wu, Yang and others},
  journal={arXiv preprint arXiv:2402.03300},
  year={2024}
}

@article{achiam2023gpt,
  title={Gpt-4 technical report},
  author={Achiam, Josh and Adler, Steven and Agarwal, Sandhini and Ahmad, Lama and Akkaya, Ilge and Aleman, Florencia Leoni and Almeida, Diogo and Altenschmidt, Janko and Altman, Sam and Anadkat, Shyamal and others},
  journal={arXiv preprint arXiv:2303.08774},
  year={2023}
}

@book{sutton2018reinforcement,
  title={Reinforcement learning: An introduction},
  author={Sutton, Richard S and Barto, Andrew G},
  year={2018},
  publisher={MIT press}
}

@article{rockafellar2000optimization,
  title={Optimization of conditional value-at-risk},
  author={Rockafellar, R Tyrrell and Uryasev, Stanislav and others},
  journal={Journal of risk},
  volume={2},
  pages={21--42},
  year={2000}
}

@article{chi2024unveiling,
  title={Unveiling causal reasoning in large language models: Reality or mirage?},
  author={Chi, Haoang and Li, He and Yang, Wenjing and Liu, Feng and Lan, Long and Ren, Xiaoguang and Liu, Tongliang and Han, Bo},
  journal={Advances in Neural Information Processing Systems},
  volume={37},
  pages={96640--96670},
  year={2024}
}

@article{chi2024does,
  title={Does Confusion Really Hurt Novel Class Discovery?},
  author={Chi, Haoang and Yang, Wenjing and Liu, Feng and Lan, Long and Qin, Tao and Han, Bo},
  journal={International Journal of Computer Vision},
  volume={132},
  number={8},
  pages={3191--3207},
  year={2024},
  publisher={Springer}
}

@misc{wang2025modelpredictivetasksampling,
      title={Model Predictive Task Sampling for Efficient and Robust Adaptation}, 
      author={Qi Cheems Wang and Zehao Xiao and Yixiu Mao and Yun Qu and Jiayi Shen and Yiqin Lv and Xiangyang Ji},
      year={2025},
      eprint={2501.11039},
      archivePrefix={arXiv},
      primaryClass={cs.LG},
      url={https://arxiv.org/abs/2501.11039}, 
}

@phdthesis{chow2017risk,
  title={Risk-sensitive and data-driven sequential decision making},
  author={Chow, Yinlam},
  year={2017},
  school={Stanford University}
}

@article{chow2018risk,
  title={Risk-constrained reinforcement learning with percentile risk criteria},
  author={Chow, Yinlam and Ghavamzadeh, Mohammad and Janson, Lucas and Pavone, Marco},
  journal={Journal of Machine Learning Research},
  volume={18},
  number={167},
  pages={1--51},
  year={2018}
}

@article{rigter2021risk,
  title={Risk-averse bayes-adaptive reinforcement learning},
  author={Rigter, Marc and Lacerda, Bruno and Hawes, Nick},
  journal={Advances in Neural Information Processing Systems},
  volume={34},
  pages={1142--1154},
  year={2021}
}

@inproceedings{pan2019risk,
  title={Risk averse robust adversarial reinforcement learning},
  author={Pan, Xinlei and Seita, Daniel and Gao, Yang and Canny, John},
  booktitle={2019 International Conference on Robotics and Automation (ICRA)},
  pages={8522--8528},
  year={2019},
  organization={IEEE}
}

@article{tamar2015policy,
  title={Policy gradient for coherent risk measures},
  author={Tamar, Aviv and Chow, Yinlam and Ghavamzadeh, Mohammad and Mannor, Shie},
  journal={Advances in neural information processing systems},
  volume={28},
  year={2015}
}

@article{chow2015risk,
  title={Risk-sensitive and robust decision-making: a cvar optimization approach},
  author={Chow, Yinlam and Tamar, Aviv and Mannor, Shie and Pavone, Marco},
  journal={Advances in neural information processing systems},
  volume={28},
  year={2015}
}

@article{greenberg2022efficient,
  title={Efficient risk-averse reinforcement learning},
  author={Greenberg, Ido and Chow, Yinlam and Ghavamzadeh, Mohammad and Mannor, Shie},
  journal={Advances in Neural Information Processing Systems},
  volume={35},
  pages={32639--32652},
  year={2022}
}

@inproceedings{ritter2018been,
  title={Been there, done that: Meta-learning with episodic recall},
  author={Ritter, Samuel and Wang, Jane and Kurth-Nelson, Zeb and Jayakumar, Siddhant and Blundell, Charles and Pascanu, Razvan and Botvinick, Matthew},
  booktitle={International conference on machine learning},
  pages={4354--4363},
  year={2018},
  organization={PMLR}
}

@article{tiboni2023domain,
  title={Domain Randomization via Entropy Maximization},
  author={Tiboni, Gabriele and Klink, Pascal and Peters, Jan and Tommasi, Tatiana and D'Eramo, Carlo and Chalvatzaki, Georgia},
  journal={arXiv preprint arXiv:2311.01885},
  year={2023}
}

@article{beck2023survey,
  title={A survey of meta-reinforcement learning},
  author={Beck, Jacob and Vuorio, Risto and Liu, Evan Zheran and Xiong, Zheng and Zintgraf, Luisa and Finn, Chelsea and Whiteson, Shimon},
  journal={arXiv preprint arXiv:2301.08028},
  year={2023}
}

@inproceedings{muratore2018domain,
  title={Domain randomization for simulation-based policy optimization with transferability assessment},
  author={Muratore, Fabio and Treede, Felix and Gienger, Michael and Peters, Jan},
  booktitle={Conference on Robot Learning},
  pages={700--713},
  year={2018},
  organization={PMLR}
}

@article{muratore2021data,
  title={Data-efficient domain randomization with bayesian optimization},
  author={Muratore, Fabio and Eilers, Christian and Gienger, Michael and Peters, Jan},
  journal={IEEE Robotics and Automation Letters},
  volume={6},
  number={2},
  pages={911--918},
  year={2021},
  publisher={IEEE}
}

@article{zintgraf2019varibad,
  title={Varibad: A very good method for bayes-adaptive deep rl via meta-learning},
  author={Zintgraf, Luisa and Shiarlis, Kyriacos and Igl, Maximilian and Schulze, Sebastian and Gal, Yarin and Hofmann, Katja and Whiteson, Shimon},
  journal={arXiv preprint arXiv:1910.08348},
  year={2019}
}

@article{wang2022learning,
  title={Learning expressive meta-representations with mixture of expert neural processes},
  author={Wang, Qi and Van Hoof, Herke},
  journal={Advances in neural information processing systems},
  volume={35},
  pages={26242--26255},
  year={2022}
}

@inproceedings{wang2022model,
  title={Model-based meta reinforcement learning using graph structured surrogate models and amortized policy search},
  author={Wang, Qi and Van Hoof, Herke},
  booktitle={International Conference on Machine Learning},
  pages={23055--23077},
  year={2022},
  organization={PMLR}
}

@inproceedings{xu2021enhancing,
  title={Enhancing context-based meta-reinforcement learning algorithms via an efficient task encoder (student abstract)},
  author={Xu, Feng and Jiang, Shengyi and Yin, Hao and Zhang, Zongzhang and Yu, Yang and Li, Ming and Li, Dong and Liu, Wulong},
  booktitle={Proceedings of the AAAI Conference on Artificial Intelligence},
  volume={35},
  number={18},
  pages={15937--15938},
  year={2021}
}

@inproceedings{rakelly2019efficient,
  title={Efficient off-policy meta-reinforcement learning via probabilistic context variables},
  author={Rakelly, Kate and Zhou, Aurick and Finn, Chelsea and Levine, Sergey and Quillen, Deirdre},
  booktitle={International conference on machine learning},
  pages={5331--5340},
  year={2019},
  organization={PMLR}
}

@article{zheng2025act,
  title={Act only when it pays: Efficient reinforcement learning for llm reasoning via selective rollouts},
  author={Zheng, Haizhong and Zhou, Yang and Bartoldson, Brian R and Kailkhura, Bhavya and Lai, Fan and Zhao, Jiawei and Chen, Beidi},
  journal={arXiv preprint arXiv:2506.02177},
  year={2025}
}

@inproceedings{qu2026can,
  title={Can prompt difficulty be online predicted for accelerating rl finetuning of reasoning models?},
  author={Qu, Yun and Wang, Qi and Mao, Yixiu and Hu, Vincent Tao and Ommer, Bj{\"o}rn and Ji, Xiangyang},
  booktitle={Proceedings of the 32nd ACM SIGKDD Conference on Knowledge Discovery and Data Mining V. 1},
  pages={1240--1250},
  year={2026}
}

@inproceedings{rempe2022generating,
  title={Generating useful accident-prone driving scenarios via a learned traffic prior},
  author={Rempe, Davis and Philion, Jonah and Guibas, Leonidas J and Fidler, Sanja and Litany, Or},
  booktitle={Proceedings of the IEEE/CVF Conference on Computer Vision and Pattern Recognition},
  pages={17305--17315},
  year={2022}
}

@inproceedings{carpin2016risk,
  title={Risk aversion in finite Markov Decision Processes using total cost criteria and average value at risk},
  author={Carpin, Stefano and Chow, Yin-Lam and Pavone, Marco},
  booktitle={2016 ieee international conference on robotics and automation (icra)},
  pages={335--342},
  year={2016},
  organization={IEEE}
}

@article{wang2024simple,
  title={A simple yet effective strategy to robustify the meta learning paradigm},
  author={Wang, Qi and Lv, Yiqin and Xie, Zheng and Huang, Jincai and others},
  journal={Advances in Neural Information Processing Systems},
  volume={36},
  year={2024}
}

@article{jiang2021replay,
  title={Replay-guided adversarial environment design},
  author={Jiang, Minqi and Dennis, Michael and Parker-Holder, Jack and Foerster, Jakob and Grefenstette, Edward and Rockt{\"a}schel, Tim},
  journal={Advances in Neural Information Processing Systems},
  volume={34},
  pages={1884--1897},
  year={2021}
}

@inproceedings{koprulu2023risk,
  title={Risk-aware curriculum generation for heavy-tailed task distributions},
  author={Koprulu, Cevahir and Sim{\~a}o, Thiago D and Jansen, Nils and Topcu, Ufuk},
  booktitle={Uncertainty in Artificial Intelligence},
  pages={1132--1142},
  year={2023},
  organization={PMLR}
}

@article{dennis2020emergent,
  title={Emergent complexity and zero-shot transfer via unsupervised environment design},
  author={Dennis, Michael and Jaques, Natasha and Vinitsky, Eugene and Bayen, Alexandre and Russell, Stuart and Critch, Andrew and Levine, Sergey},
  journal={Advances in neural information processing systems},
  volume={33},
  pages={13049--13061},
  year={2020}
}

@article{rajeswaran2019meta,
  title={Meta-learning with implicit gradients},
  author={Rajeswaran, Aravind and Finn, Chelsea and Kakade, Sham M and Levine, Sergey},
  journal={Advances in neural information processing systems},
  volume={32},
  year={2019}
}

@inproceedings{abbas2022sharp,
  title={Sharp-maml: Sharpness-aware model-agnostic meta learning},
  author={Abbas, Momin and Xiao, Quan and Chen, Lisha and Chen, Pin-Yu and Chen, Tianyi},
  booktitle={International conference on machine learning},
  pages={10--32},
  year={2022},
  organization={PMLR}
}

@article{finn2018probabilistic,
  title={Probabilistic model-agnostic meta-learning},
  author={Finn, Chelsea and Xu, Kelvin and Levine, Sergey},
  journal={Advances in neural information processing systems},
  volume={31},
  year={2018}
}

@inproceedings{wang2022bridge,
  title={Bridge the inference gaps of neural processes via expectation maximization},
  author={Wang, Qi and Federici, Marco and van Hoof, Herke},
  booktitle={The Eleventh International Conference on Learning Representations},
  year={2022}
}

@article{greenberg2024train,
  title={Train hard, fight easy: Robust meta reinforcement learning},
  author={Greenberg, Ido and Mannor, Shie and Chechik, Gal and Meirom, Eli},
  journal={Advances in Neural Information Processing Systems},
  volume={36},
  year={2024}
}

@article{hospedales2021meta,
  title={Meta-learning in neural networks: A survey},
  author={Hospedales, Timothy and Antoniou, Antreas and Micaelli, Paul and Storkey, Amos},
  journal={IEEE transactions on pattern analysis and machine intelligence},
  volume={44},
  number={9},
  pages={5149--5169},
  year={2021},
  publisher={IEEE}
}

@inproceedings{
lv2024theoretical,
title={Theoretical Investigations and Practical Enhancements on Tail Task Risk Minimization in Meta Learning},
author={Yiqin Lv and Cheems Wang and Dong Liang and Zheng Xie},
booktitle={The Thirty-eighth Annual Conference on Neural Information Processing Systems},
year={2024},
url={https://openreview.net/forum?id=McrzOo0hwr}
}

@article{ouyang2022training,
  title={Training language models to follow instructions with human feedback},
  author={Ouyang, Long and Wu, Jeffrey and Jiang, Xu and Almeida, Diogo and Wainwright, Carroll and Mishkin, Pamela and Zhang, Chong and Agarwal, Sandhini and Slama, Katarina and Ray, Alex and others},
  journal={Advances in neural information processing systems},
  volume={35},
  pages={27730--27744},
  year={2022}
}

@inproceedings{fujimoto2018addressing,
  title={Addressing function approximation error in actor-critic methods},
  author={Fujimoto, Scott and Hoof, Herke and Meger, David},
  booktitle={International conference on machine learning},
  pages={1587--1596},
  year={2018},
  organization={PMLR}
}

@article{kingma2013auto,
  title={Auto-encoding variational bayes},
  author={Kingma, Diederik P and Welling, Max},
  journal={arXiv preprint arXiv:1312.6114},
  year={2013}
}

@article{collins2020task,
  title={Task-robust model-agnostic meta-learning},
  author={Collins, Liam and Mokhtari, Aryan and Shakkottai, Sanjay},
  journal={Advances in Neural Information Processing Systems},
  volume={33},
  pages={18860--18871},
  year={2020}
}

@article{bai2022training,
  title={Training a helpful and harmless assistant with reinforcement learning from human feedback},
  author={Bai, Yuntao and Jones, Andy and Ndousse, Kamal and Askell, Amanda and Chen, Anna and DasSarma, Nova and Drain, Dawn and Fort, Stanislav and Ganguli, Deep and Henighan, Tom and others},
  journal={arXiv preprint arXiv:2204.05862},
  year={2022}
}

@article{vapnik1998statistical,
  title={Statistical learning theory},
  author={Vapnik, Vladimir Naumovich},
  year={1998},
  publisher={wiley New York}
}

@inproceedings{sagawa2019distributionally,
  title={Distributionally Robust Neural Networks},
  author={Sagawa, Shiori and Koh, Pang Wei and Hashimoto, Tatsunori B and Liang, Percy},
  booktitle={International Conference on Learning Representations},
  year={2019}
}

@article{mnih2015human,
  title={Human-level control through deep reinforcement learning},
  author={Mnih, Volodymyr and Kavukcuoglu, Koray and Silver, David and Rusu, Andrei A and Veness, Joel and Bellemare, Marc G and Graves, Alex and Riedmiller, Martin and Fidjeland, Andreas K and Ostrovski, Georg and others},
  journal={nature},
  volume={518},
  number={7540},
  pages={529--533},
  year={2015},
  publisher={Nature Publishing Group}
}

@inproceedings{tobin2017domain,
  title={Domain randomization for transferring deep neural networks from simulation to the real world},
  author={Tobin, Josh and Fong, Rachel and Ray, Alex and Schneider, Jonas and Zaremba, Wojciech and Abbeel, Pieter},
  booktitle={2017 IEEE/RSJ international conference on intelligent robots and systems (IROS)},
  pages={23--30},
  year={2017},
  organization={IEEE}
}

@article{wang2024robust,
  title={Robust Fast Adaptation from Adversarially Explicit Task Distribution Generation},
  author={Wang, Cheems and Lv, Yiqin and Mao, Yixiu and Qu, Yun and Xu, Yi and Ji, Xiangyang},
  journal={arXiv preprint arXiv:2407.19523},
  year={2024}
}

@article{qu2025fast,
  title={Fast and Robust: Task Sampling with Posterior and Diversity Synergies for Adaptive Decision-Makers in Randomized Environments},
  author={Qu, Yun and Wang, Qi Cheems and Mao, Yixiu and Lv, Yiqin and Ji, Xiangyang},
  journal={arXiv preprint arXiv:2504.19139},
  year={2025}
}

@misc{wei2023cot,
      title={Chain-of-Thought Prompting Elicits Reasoning in Large Language Models}, 
      author={Jason Wei and Xuezhi Wang and Dale Schuurmans and Maarten Bosma and Brian Ichter and Fei Xia and Ed Chi and Quoc Le and Denny Zhou},
      year={2023},
      eprint={2201.11903},
      archivePrefix={arXiv},
      primaryClass={cs.CL},
      url={https://arxiv.org/abs/2201.11903}, 
}

@misc{variationalinference,
      title={An Introduction to Variational Inference}, 
      author={Ankush Ganguly and Samuel W. F. Earp},
      year={2021},
      eprint={2108.13083},
      archivePrefix={arXiv},
      primaryClass={cs.LG},
      url={https://arxiv.org/abs/2108.13083}, 
}

@misc{yangqwen3,
      title={Qwen3 Technical Report}, 
      author={An Yang and Anfeng Li and Baosong Yang and others},
      year={2025},
      eprint={2505.09388},
      archivePrefix={arXiv},
      primaryClass={cs.CL},
      url={https://arxiv.org/abs/2505.09388}, 
}

@misc{bae2026online,
      title={Online Difficulty Filtering for Reasoning Oriented Reinforcement Learning}, 
      author={Sanghwan Bae and Jiwoo Hong and Min Young Lee and Hanbyul Kim and JeongYeon Nam and Donghyun Kwak},
      year={2026},
      eprint={2504.03380},
      archivePrefix={arXiv},
      primaryClass={cs.CL},
      url={https://arxiv.org/abs/2504.03380}, 
}

@inproceedings{hokoff,
  title={Hokoff: Real Game Dataset from Honor of Kings and its Offline Reinforcement Learning Benchmarks},
  author={Qu, Yun and Wang, Boyuan and Shao, Jianzhun and others},
  booktitle={NeurIPS},
  year={2023}
}

@article{cfcql,
  title={Counterfactual conservative Q learning for offline multi-agent reinforcement learning},
  author={Shao, Jianzhun and Qu, Yun and Chen, Chen and Zhang, Hongchang and Ji, Xiangyang},
  journal={NeurIPS},
  volume={36},
  year={2024}
}

@inproceedings{wang2025robust,
  title={Robust fast adaptation from adversarially explicit task distribution generation},
  author={Wang, Qi and Lv, Yiqin and Mao, Yixiu and Qu, Yun and Xu, Yi and Ji, Xiangyang},
  booktitle={Proceedings Of The 31st ACM SIGKDD Conference On Knowledge Discovery And Data Mining V. 1},
  pages={1481--1491},
  year={2025}
}

@inproceedings{dps,
  title={Dynamics-predictive sampling for active rl finetuning of large reasoning models},
  author={Mao, Yixiu and Qu, Yun and Wang, Qi and Zou, Heming and Ji, Xiangyang},
  booktitle={International Conference on Learning Representations},
  volume={2026},
  pages={25242--25270},
  year={2026}
}

@article{gps,
  title={Small generalizable prompt predictive models can steer efficient rl post-training of large reasoning models},
  author={Qu, Yun and Wang, Qi and Mao, Yixiu and Zou, Heming and Jiang, Yuhang and Liu, Weijie and Bai, Clive and Yang, Kai and Chen, Yangkun and Yang, Saiyong and others},
  journal={arXiv preprint arXiv:2602.01970},
  year={2026}
}

@article{trace,
  title={TRACE: A Unified Rollout Budget Allocation Framework for Efficient Agentic Reinforcement Learning},
  author={Zou, Heming and Wang, Qi and Qu, Yun and Jiang, Yuhang and Cai, Lizhou and Mao, Yixiu and Peng, Ru and Xu, Xin and Liu, Weijie and Yang, Kai and others},
  journal={arXiv preprint arXiv:2606.11119},
  year={2026}
}

@article{popo,
  title={RLVR without Ineffective Samples: Group Prioritized Off-Policy Optimization for LLM Reasoning},
  author={Mao, Yixiu and Qu, Yun and Wang, Qi and Zou, Heming and Ji, Xiangyang},
  journal={arXiv preprint arXiv:2606.01281},
  year={2026}
}

@article{lpo,
  title={Listwise Policy Optimization: Group-based RLVR as Target-Projection on the LLM Response Simplex},
  author={Qu, Yun and Wang, Qi and Mao, Yixiu and Zou, Heming and Jiang, Yuhang and Li, Yingyue and Xu, Wutong and Cai, Lizhou and Liu, Weijie and Bai, Clive and others},
  journal={arXiv preprint arXiv:2605.06139},
  year={2026}
}

@InProceedings{pmlr-v267-li25g,
  title = 	 {Transformer-Based Spatial-Temporal Counterfactual Outcomes Estimation},
  author =       {Li, He and Chi, Haoang and Liu, Mingyu and Huang, Wanrong and Xu, Liyang and Yang, Wenjing},
  booktitle = 	 {International Conference on Machine Learning},
  pages = 	 {34185--34208},
  year = 	 {2025},
  volume = 	 {267},
}

@article{yang2025nt,
  title={NT-FAN: A simple yet effective noise-tolerant few-shot adaptation network},
  author={Yang, Wenjing and Chi, Haoang and Zhan, Yibing and Hu, Bowen and Ren, Xiaoguang and Tao, Dapeng and Lan, Long},
  journal={Artificial Intelligence},
  volume={346},
  pages={104363},
  year={2025},
  publisher={Elsevier}
}

@inproceedings{chi2022meta,
title={Meta Discovery: Learning to Discover Novel Classes given Very Limited Data},
author={Haoang Chi and Feng Liu and Wenjing Yang and Long Lan and Tongliang Liu and Bo Han and Gang Niu and Mingyuan Zhou and Masashi Sugiyama},
booktitle={International Conference on Learning Representations},
year={2022}
}

\appendix

\section{Related Work}
\label{appendix:related_work}

\paragraph{Risk-sensitive optimization for adaptive sequential decision-making.}
Risk considerations have long been central to sequential decision-making, where poor adaptation may lead to severe downstream failures~\citep{chow2015risk, chow2017risk, chow2018risk,taolatent2026,pmlr-v267-li25g}. This has motivated a broad line of work on incorporating risk-averse objectives into policy optimization~\citep{tamar2015policy, pan2019risk, rigter2021risk, greenberg2022efficient}. In randomized environments, including DR~\citep{tobin2017domain, muratore2018domain} and Meta-RL~\citep{finn2017model, beck2023survey,wang2022learning}, several optimization principles have been adopted to improve robustness. Empirical risk minimization (ERM)~\citep{vapnik1998statistical} optimizes the average loss under the training task distribution and is therefore commonly used as a standard baseline. Distributionally robust risk minimization (DRM) instead emphasizes high-risk tasks via conditional value-at-risk (CVaR) objectives~\citep{rockafellar2000optimization, wang2024simple, lv2024theoretical, greenberg2024train}, while group distributionally robust optimization (GDRM) mitigates subpopulation shifts by reweighting group-level risks~\citep{sagawa2019distributionally}. Worst-case meta-learning has also been investigated~\citep{collins2020task}, but typically requires careful relaxations to maintain stable training. A shared limitation of these approaches is their reliance on extensive policy evaluation over large task spaces, which makes purely evaluation-driven robust optimization expensive in practice.

\paragraph{Robust active task sampling and risk-predictive models.}
Robust active task sampling (RATS) addresses the above efficiency bottleneck by replacing repeated exhaustive evaluation with a learned RPM~\citep{wang2025modelpredictivetasksampling}. Instead of executing the current policy on every candidate task, RATS estimates task risk through the RPM and then selects an informative or high-risk subset for policy optimization. MPTS~\citep{wang2025modelpredictivetasksampling} instantiates this idea by modeling the optimization trajectory as a generative sequence and using streaming variational inference to predict adaptation risk on unseen tasks. More recently, PDTS~\citep{qu2025fast} reinterprets RATS from a task-selection perspective, formulates subset selection as an infinitely many-armed bandit problem, and combines posterior sampling with diversity regularization to alleviate concentration under large candidate batches. These methods substantially improve the efficiency of robust task selection; however, their predictive modules remain challenged by two factors: early-stage histories provide only sparse supervision, and risk estimates are inherently delayed because predictions are made before the next policy update is observed. Our method focuses on these RPM-level limitations by introducing in-context sequence modeling and a one-step look-ahead prediction mechanism.

\paragraph{Cross-task adaptation and curriculum learning.}
Zero-shot and few-shot adaptation are two central evaluation regimes for adaptive decision-makers. DR aims to obtain zero-shot generalization by training over randomized environments~\citep{tobin2017domain, muratore2018domain, mehta2020active, muratore2021data, tiboni2023domain, yang2025nt, chi2022meta, chi2024does, chi2024unveiling}, whereas meta-learning seeks reusable priors that enable rapid adaptation from limited experience~\citep{hospedales2021meta}. Representative Meta-RL methods include optimization-based algorithms that learn initializations~\citep{finn2017model, finn2018probabilistic, abbas2022sharp, rajeswaran2019meta}, context-based approaches that infer latent task representations from support trajectories~\citep{rakelly2019efficient, zintgraf2019varibad, wang2022bridge, xu2021enhancing}, and recurrent architectures that internalize adaptation through memory~\citep{ritter2018been}. Curriculum learning provides a complementary mechanism by changing which tasks are presented during training. For example, unsupervised environment design and its variants generate tasks with increasing challenge~\citep{dennis2020emergent, jiang2021replay, koprulu2023risk}, while adversarial task-distribution methods construct curricula through learned generative distributions~\citep{wang2024robust}. Our work is orthogonal to such curriculum-design strategies: rather than modifying the task distribution itself, we improve the accuracy and timeliness of risk estimation under a given distribution, making the approach potentially compatible with existing curricula.

\paragraph{Reinforcement learning for LLM post-training.}
Reinforcement learning from human feedback (RLHF) has become a standard approach for aligning large language models with human preferences~\citep{ouyang2022training, bai2022training}. Recent post-training pipelines further use reinforcement learning to improve reasoning and task-solving capabilities~\citep{mao2026rlvr,qu2026listwise}, where prompts can be viewed as task instances and responses are scored by reward models or verifiers~\citep{shao2024deepseekmath}. Prompt-level curricula that emphasize informative or challenging examples have also been explored to improve training efficiency~\citep{gao2026prompt,popo,trace}. A practical obstacle is that estimating prompt difficulty usually requires costly rollouts from the current model. Our framework can be naturally mapped to this setting by treating prompts as task identifiers and using the RPM to estimate prompt-level difficulty, thereby supporting curriculum construction for LLM post-training without additional generation overhead.

\section{Expanded Proofs for Theoretical Analysis}
\label{app:proofs}

This appendix contains the full proofs of all theoretical results stated in the main text.
We first collect notation, then prove Theorem~\ref{thm:vanilla_bound} and Proposition~\ref{prop:lookahead-mismatch}.

\paragraph{Notation.}
$\mathcal{T}$: task space; $\Theta$: parameter space; $\boldsymbol{\theta}_t \in \Theta$: agent parameters at step $t$; $L(\boldsymbol{\theta}, \boldsymbol{\tau})$: true risk of task $\boldsymbol{\tau} \in \mathcal{T}$ under $\boldsymbol{\theta}$; $H_t = \{(\boldsymbol{\tau}_{t,i}, \ell_{t,i})\}_{i=1}^{B}$: per-iteration history with batch size $B$; $H_{0:t} = \bigcup_{k=0}^{t} H_k$: cumulative trajectory; $\mathcal{F}$: RPM hypothesis class; $\mathfrak{R}_n(\mathcal{F})$: Rademacher complexity with $n$ samples.
The vanilla RPM $\hat{\ell}_{\mathrm{vanilla}}$ is trained on $H_t$ (size $B$).
ICRPM $\hat{\ell}_{\mathrm{ICRPM}} = f_\phi(\boldsymbol{\tau}, H_{0:t})$ is trained on $W$ consecutive batches (size $N = WB$).

\noindent\textbf{Theorem~\ref{thm:vanilla_bound}}
(Per-step Lagged Prediction Bound for Vanilla RPMs) 
Assume that $\ell(\bm{\tau})=\mathcal{L}(\bm{\theta},\bm{\tau})\in[0,M]$ for all
$\bm{\theta}\in\Theta$, $\bm{\tau}\in\mathcal T$, and that
$\ell(\bm{\tau})$ is $L$-Lipschitz in $\bm{\theta}$. The policy update uses a normalized average loss computed by i.i.d. samples and
$\|\nabla_{\bm{\theta}}\ell_t(\bm{\tau}_i)\|\le G$.
Let $\hat\ell_{\bm{\psi},\bm{\phi},t}(\bm{\tau})$ denote the predictive risk
function induced by the vanilla RPM at iteration $t$. Define
$
    \epsilon_{\mathrm{app}}^{\mathrm{rec}}:=
    \inf_{\bm{\psi},\bm{\phi}}
    \mathbb E_{\bm{\tau}\sim p}
    \bigl[
        \bigl(\hat\ell_{\bm{\psi},\bm{\phi},t}(\bm{\tau})
        -\ell_t(\bm{\tau})\bigr)^2
    \bigr].
$
Then, for any $\delta\in(0,1)$, with probability at least $1-\delta$ over the
sampling of $H_t$,
\begin{equation}
    \mathbb E_{\bm{\tau}\sim p}
    \bigl[|\hat\ell_t(\bm{\tau})-\ell_{t+1}(\bm{\tau})|\bigr]
    \le
    \underbrace{
    \sqrt{\epsilon_{\mathrm{app}}^{\mathrm{rec}}
    +\epsilon_{\mathrm{gen}}^{\mathrm{rec}}(B,\delta)}
    }_{\textcircled{\small{1}}\text{Reconstruction Approximation Error}}
    \qquad
    +
    \underbrace{
    L\eta G
    }_{\textcircled{\small{2}}\text{Temporal Lagging Error}} .
\end{equation}
If the reconstruction loss class
$
    \mathcal G^{\mathrm{rec}}
    =
    \{(\bm{\tau},y)\mapsto
    (\hat\ell_{\bm{\psi},\bm{\phi},t}(\bm{\tau})-y)^2:
    \bm{\psi}\in\Psi,\bm{\phi}\in\Phi\}
$
has Rademacher complexity $\mathfrak R_B(\mathcal G^{\mathrm{rec}})$ and is bounded by $M^2$, then
$
\epsilon_{\mathrm{gen}}^{\mathrm{rec}}(B,\delta)
=
4\mathfrak R_B(\mathcal G^{\mathrm{rec}})
+
2M^2\sqrt{\frac{\ln(2/\delta)}{2B}}.
$

\begin{proof}
Assume the RPM uses a fixed-variance Gaussian decoder
$p_{\bm{\psi}}(\ell\vert\bm{\tau},\bm{z})=\mathcal N(\mu_{\bm{\psi}}(\bm{\tau},\bm{z}),\sigma^2)$,
and define the predictive mean as
$\hat\ell_{\bm{\psi},\bm{\phi},t}(\bm{\tau})
:=
\mathbb E_{q_{\bm{\phi}}(\bm{z}_t\vert H_t)}
[\mu_{\bm{\psi}}(\bm{\tau},\bm{z}_t)]$.
For such a decoder, the reconstruction term in the negative log-likelihood is proportional to squared reconstruction error:
\begin{equation}
    -\log p_{\bm{\psi}}(\ell_t(\bm{\tau})\vert\bm{\tau},\bm{z}_t)
    =
    \frac{1}{2\sigma^2}
    \bigl(\mu_{\bm{\psi}}(\bm{\tau},\bm{z}_t)-\ell_t(\bm{\tau})\bigr)^2
    + C_{\sigma},
\end{equation}
where $C_{\sigma}$ is independent of $(\bm{\psi},\bm{\phi})$.
Therefore, the reconstruction component of the ELBO in Eq.~(2) induces the squared reconstruction objective analyzed below, while the KL term acts as a latent regularizer and is absorbed into the approximation and optimization error.

Let
$
    f_{\bm{\psi},\bm{\phi},t}(\bm{\tau})
    :=
    \hat\ell_{\bm{\psi},\bm{\phi},t}(\bm{\tau})
$
and define the empirical and population squared reconstruction risks as
\begin{equation}
    \widehat{\mathcal R}_t(f_{\bm{\psi},\bm{\phi},t})
    :=
    \frac{1}{B}
    \sum_{i=1}^B
    \bigl(
        f_{\bm{\psi},\bm{\phi},t}(\bm{\tau}_{t,i})
        -
        \ell_t(\bm{\tau}_{t,i})
    \bigr)^2 ,
\end{equation}
and
\begin{equation}
    \mathcal R_t(f_{\bm{\psi},\bm{\phi},t})
    :=
    \mathbb E_{\bm{\tau}\sim p}
    \bigl[
        \bigl(
        f_{\bm{\psi},\bm{\phi},t}(\bm{\tau})
        -
        \ell_t(\bm{\tau})
        \bigr)^2
    \bigr] .
\end{equation}
By the standard Rademacher uniform convergence bound applied to
$\mathcal G^{\mathrm{rec}}$, with probability at least $1-\delta$, uniformly over
$(\bm{\psi},\bm{\phi})\in\Psi\times\Phi$,
\begin{equation}
    \mathcal R_t(f_{\bm{\psi},\bm{\phi},t})
    \le
    \widehat{\mathcal R}_t(f_{\bm{\psi},\bm{\phi},t})
    +
    2\mathfrak R_B(\mathcal G^{\mathrm{rec}})
    +
    M^2\sqrt{\frac{\ln(2/\delta)}{2B}} .
\end{equation}
Let $(\hat{\bm{\psi}},\hat{\bm{\phi}})$ denote the parameters obtained by minimizing the empirical squared reconstruction objective, and let
$f^*_{\bm{\psi},\bm{\phi},t}$ be a population minimizer in the class.
Applying the above inequality to the empirical minimizer gives
\begin{equation}
\begin{aligned}
    \mathcal R_t(\hat\ell_t)
    &\le
    \widehat{\mathcal R}_t(\hat\ell_t)
    +
    2\mathfrak R_B(\mathcal G^{\mathrm{rec}})
    +
    M^2\sqrt{\frac{\ln(2/\delta)}{2B}}  \\
    &\le
    \widehat{\mathcal R}_t(f^*_{\bm{\psi},\bm{\phi},t})
    +
    2\mathfrak R_B(\mathcal G^{\mathrm{rec}})
    +
    M^2\sqrt{\frac{\ln(2/\delta)}{2B}} .
\end{aligned}
\end{equation}
Applying the uniform convergence bound once more to
$f^*_{\bm{\psi},\bm{\phi},t}$ yields
\begin{equation}
    \widehat{\mathcal R}_t(f^*_{\bm{\psi},\bm{\phi},t})
    \le
    \mathcal R_t(f^*_{\bm{\psi},\bm{\phi},t})
    +
    2\mathfrak R_B(\mathcal G^{\mathrm{rec}})
    +
    M^2\sqrt{\frac{\ln(2/\delta)}{2B}} .
\end{equation}
Combining the last two inequalities and using the definition of
$\epsilon_{\mathrm{app}}^{\mathrm{rec}}$ gives
\begin{equation}
\begin{aligned}
    \mathcal R_t(\hat\ell_t)
    &\le
    \epsilon_{\mathrm{app}}^{\mathrm{rec}}
    +
    4\mathfrak R_B(\mathcal G^{\mathrm{rec}})
    +
    2M^2\sqrt{\frac{\ln(2/\delta)}{2B}}  \\
    &=
    \epsilon_{\mathrm{app}}^{\mathrm{rec}}
    +
    \epsilon_{\mathrm{gen}}^{\mathrm{rec}}(B,\delta).
\end{aligned}
\end{equation}
Thus,
\begin{equation}
    \mathbb E_{\bm{\tau}\sim p}
    \bigl[
        (\hat\ell_t(\bm{\tau})-\ell_t(\bm{\tau}))^2
    \bigr]
    \le
    \epsilon_{\mathrm{app}}^{\mathrm{rec}}
    +
    \epsilon_{\mathrm{gen}}^{\mathrm{rec}}(B,\delta).
\end{equation}

It remains to account for the temporal shift from $\bm{\theta}_t$ to
$\bm{\theta}_{t+1}$. By the $L_2$ triangle inequality and Jensen's inequality,
\begin{equation}
\begin{aligned}
    &\mathbb E_{\bm{\tau}\sim p}
    \bigl[|\hat\ell_t(\bm{\tau})-\ell_{t+1}(\bm{\tau})|\bigr] \\
    &\le
    \bigl(
    \mathbb E_{\bm{\tau}\sim p}
    [
    (\hat\ell_t(\bm{\tau})-\ell_{t+1}(\bm{\tau}))^2
    ]
    \bigr)^{1/2} \\
    &\le
    \bigl(
    \mathbb E_{\bm{\tau}\sim p}
    [
    (\hat\ell_t(\bm{\tau})-\ell_t(\bm{\tau}))^2
    ]
    \bigr)^{1/2}
    +
    \bigl(
    \mathbb E_{\bm{\tau}\sim p}
    [
    (\ell_t(\bm{\tau})-\ell_{t+1}(\bm{\tau}))^2
    ]
    \bigr)^{1/2}.
\end{aligned}
\end{equation}
The first term is controlled by the reconstruction bound above. For the second term, since $\ell(\bm{\tau})$ is $L$-Lipschitz in $\bm{\theta}$,
\begin{equation}
    |\ell_t(\bm{\tau})-\ell_{t+1}(\bm{\tau})|
    \le
    L\|\bm{\theta}_{t+1}-\bm{\theta}_t\|.
\end{equation}
The normalized policy update and the gradient bound imply
\begin{equation}
\begin{aligned}
    \|\bm{\theta}_{t+1}-\bm{\theta}_t\|
    &=
    \left\|
    \eta\frac{1}{B}
    \sum_{i=1}^B
    \nabla_{\bm{\theta}}\ell_t(\bm{\tau}_i)
    \right\|  \\
    &\le
    \eta\frac{1}{B}
    \sum_{i=1}^B
    \|\nabla_{\bm{\theta}}\ell_t(\bm{\tau}_i)\|
    \le
    \eta G .
\end{aligned}
\end{equation}
Therefore,
\begin{equation}
    \bigl(
    \mathbb E_{\bm{\tau}\sim p}
    [
    (\ell_t(\bm{\tau})-\ell_{t+1}(\bm{\tau}))^2
    ]
    \bigr)^{1/2}
    \le
    L\eta G .
\end{equation}
Substituting the two bounds gives
\begin{equation}
\begin{aligned}
    \mathbb E_{\bm{\tau}\sim p}
    \bigl[|\hat\ell_t(\bm{\tau})-\ell_{t+1}(\bm{\tau})|\bigr]
    &\le
    \sqrt{
        \epsilon_{\mathrm{app}}^{\mathrm{rec}}
        +
        \epsilon_{\mathrm{gen}}^{\mathrm{rec}}(B,\delta)
    }
    +
    L\eta G ,
\end{aligned}
\end{equation}
which completes the proof.
\end{proof}

\subsection{Proof of Proposition~\ref{prop:lookahead-mismatch}}
\noindent\textbf{Proposition~\ref{prop:lookahead-mismatch}
(Look-ahead reduces temporal mismatch, restated).}
\emph{
Let $\bm{\Delta}_t=\bm{\theta}_{t+1}-\bm{\theta}_t$ with
$\|\bm{\Delta}_t\|\leq \eta G$, and assume that
$\ell(\bm{\theta},\bm{\tau})$ is $\beta$-smooth in $\bm{\theta}$.
For analysis, define the first-order look-ahead risk as
$m_t(\bm{\tau})=\ell_t(\bm{\tau})+
\nabla_{\bm{\theta}}\ell(\bm{\theta}_t,\bm{\tau})^\top\bm{\Delta}_t$.
If the ICRPM predictor satisfies
$\mathbb{E}_{\bm{\tau}\sim p}
|\hat{\ell}^{\mathrm{IC}}_{t+1}(\bm{\tau})-m_t(\bm{\tau})|
\leq\epsilon_{\mathrm{IC}}$, then
}
\begin{equation}
\mathbb{E}_{\bm{\tau}\sim p}
\left[
|\hat{\ell}^{\mathrm{IC}}_{t+1}(\bm{\tau})-\ell_{t+1}(\bm{\tau})|
\right]
\leq
\epsilon_{\mathrm{IC}}
+
\frac{\beta}{2}\eta^2G^2 .
\end{equation}
\begin{proof}
For any fixed task $\bm{\tau}$, denote
$\ell_{t+1}(\bm{\tau})=\ell(\bm{\theta}_{t+1},\bm{\tau})$ and
$\ell_t(\bm{\tau})=\ell(\bm{\theta}_t,\bm{\tau})$. Since
$\ell(\bm{\theta},\bm{\tau})$ is $\beta$-smooth in $\bm{\theta}$, the standard first-order
remainder bound gives
\begin{equation}
\left|
\ell_{t+1}(\bm{\tau})
-
\ell_t(\bm{\tau})
-
\nabla_{\bm{\theta}}\ell(\bm{\theta}_t,\bm{\tau})^\top \bm{\Delta}_t
\right|
\leq
\frac{\beta}{2}\|\bm{\Delta}_t\|^2
\leq
\frac{\beta}{2}\eta^2G^2 .
\end{equation}
By the definition of $m_t(\bm{\tau})$, this implies
$|\ell_{t+1}(\bm{\tau})-m_t(\bm{\tau})|
\leq \frac{\beta}{2}\eta^2G^2$. Therefore, using the triangle inequality and taking expectation
over $\bm{\tau}\sim p$,
\begin{equation}
\mathbb{E}_{\bm{\tau}\sim p}
\left[
|\hat{\ell}^{\mathrm{IC}}_{t+1}(\bm{\tau})-\ell_{t+1}(\bm{\tau})|
\right]
\leq
\mathbb{E}_{\bm{\tau}\sim p}
\left[
|\hat{\ell}^{\mathrm{IC}}_{t+1}(\bm{\tau})-m_t(\bm{\tau})|
\right]
+
\frac{\beta}{2}\eta^2G^2
\leq
\epsilon_{\mathrm{IC}}
+
\frac{\beta}{2}\eta^2G^2 .
\end{equation}
This completes the proof.
\end{proof}

\section{Additional Preliminaries and Training Protocols}
\label{app:unified_preliminaries}

This appendix summarizes the task formulations used in our experiments, covering
zero-shot domain randomization, few-shot meta-reinforcement learning, and RL
post-training with verifiable rewards. Although these settings differ in their
interaction modalities, they share a common structure: at each iteration, a subset
of tasks or prompts is selected, evaluated by the current policy, and appended to
the optimization history for subsequent predictive sampling.

\subsection{Task Distributions in Randomized Environments}
\label{app:task_distribution}

\begin{definition}[Identifier-induced MDP family]
\label{def:identifier_mdp}
Let $p(\bm{\tau})$ be a distribution over task identifiers. Each identifier
$\bm{\tau}$ specifies an MDP
$
\mathcal{M}_{\bm{\tau}}
=
\langle
\mathcal{S},
\mathcal{A},
P_{\bm{\tau}},
R_{\bm{\tau}},
\bar{\lambda}
\rangle$, where $\mathcal{S}$ and $\mathcal{A}$ are shared state and action spaces, while
the transition kernel $P_{\bm{\tau}}(\bm{s}'\mid \bm{s},\bm{a})$ and/or the
reward function $R_{\bm{\tau}}(\bm{s},\bm{a})$ may vary with $\bm{\tau}$. For an
$H$-horizon trajectory
$\xi=(\bm{s}_0,\bm{a}_0,r_0,\ldots,\bm{s}_{H})$, the discounted return is
\begin{equation}
G_{\bm{\tau}}(\xi)=\sum_{h=0}^{H-1}\bar{\lambda}^{h}r_h .
\end{equation}
\end{definition}

This formulation covers continuous-control tasks in robotics, where
$\bm{\tau}$ may encode physical parameters, target velocities, masses, lighting
conditions, or goal positions. It also provides the task-level abstraction used by
robust task sampling, where the sampler selects identifiers whose predicted
difficulty is informative for policy improvement.

\subsection{Zero-Shot and Few-Shot Adaptive Decision-Making}
\label{app:zero_few_shot}

\paragraph{Domain randomization.}
In zero-shot domain randomization, the policy is trained over MDPs sampled from
$p(\bm{\tau})$ and is directly evaluated on unseen environments without an
adaptation phase. The standard objective is
\begin{equation}
\label{eq:dr_objective_appendix}
\max_{\bm{\theta}}
\;
\mathcal{J}_{\mathrm{DR}}(\bm{\theta})
=
\mathbb{E}_{\bm{\tau}\sim p(\bm{\tau})}
\mathbb{E}_{\xi\sim\pi_{\bm{\theta}},\mathcal{M}_{\bm{\tau}}}
\left[
\sum_{h=0}^{H-1}\bar{\lambda}^{h}r_h
\right].
\end{equation}
For risk-based sampling, we use the negative return as the task risk,
\begin{equation}
\ell(\bm{\tau};\bm{\theta})
=
-
\mathbb{E}_{\xi\sim\pi_{\bm{\theta}},\mathcal{M}_{\bm{\tau}}}
\left[
\sum_{h=0}^{H-1}\bar{\lambda}^{h}r_h
\right].
\end{equation}
In our experiments, TD3 is used for physical domain randomization
\citep{fujimoto2018addressing}.

\paragraph{Meta-RL with MAML.}
For few-shot adaptation, we use MAML \citep{finn2017model} as the meta-learning
backbone. Each task $\bm{\tau}$ provides support trajectories
$\mathcal{D}_{\bm{\tau}}^{S}$ for fast adaptation and query trajectories
$\mathcal{D}_{\bm{\tau}}^{Q}$ for post-adaptation evaluation. The inner update is
\begin{equation}
\bm{\theta}'_{\bm{\tau}}
=
\bm{\theta}
-
\eta_{\mathrm{in}}
\nabla_{\bm{\theta}}
\ell(\mathcal{D}_{\bm{\tau}}^{S};\bm{\theta}),
\end{equation}
and the meta-objective is
\begin{equation}
\label{eq:maml_objective_appendix}
\min_{\bm{\theta}}
\;
\mathbb{E}_{\bm{\tau}\sim p(\bm{\tau})}
\left[
\ell(
\mathcal{D}_{\bm{\tau}}^{Q};
\bm{\theta}'_{\bm{\tau}}
)
\right].
\end{equation}
Here, $\ell$ is instantiated as the negative query return after fast adaptation.
Thus, DR can be viewed as the zero-shot case with no support set, while Meta-RL
uses $\mathcal{D}_{\bm{\tau}}^{S}$ before evaluating
$\mathcal{D}_{\bm{\tau}}^{Q}$.

\subsection{RL Post-Training with Verifiable Rewards}
\label{app:rlvr_prompt}

We also consider RL post-training for LLM reasoning. Here, each identifier
$\bm{\tau}$ denotes a prompt, or its embedding representation, sampled from a
prompt pool $\mathcal{T}$. At iteration $t$, the policy
$\pi_{\bm{\theta}_t}$ generates $K$ responses
$\bm{y}^{\bm{\tau}}_t=\{y^{\bm{\tau}}_{t,i}\}_{i=1}^{K}$ for each selected prompt.
A verifier assigns rewards
$\bm{r}^{\bm{\tau}}_t=\{r(\bm{\tau},y^{\bm{\tau}}_{t,i})\}_{i=1}^{K}$.
The latent success rate is
\begin{equation}
\rho_t^{\bm{\tau}}
=
\mathbb{E}[r_t^{\bm{\tau}}],
\end{equation}
which measures the probability that the current policy solves prompt
$\bm{\tau}$. The optimization history is written as
\[
H_t
=
\left\{
\mathcal{T}_t^{B},
\mathcal{R}_t^{B}
\right\},
\quad
\mathcal{R}_t^{B}
=
\left\{
\bm{r}^{\bm{\tau}_{t,j}}_t
\right\}_{j=1}^{B}, 
\]
where each \(\bm r_i^{\bm\tau_{i,j}}\) denotes the reward vector over \(K\) responses for prompt \(\bm\tau_{i,j}\), and \(B\) is the number of selected prompts in the batch.
RLVR optimizes the expected verifiable reward:
\begin{equation}
\label{eq:rlvr_objective_appendix}
\max_{\bm{\theta}}
\;
\mathbb{E}_{\bm{\tau}\sim\mathcal{T},\,
y\sim\pi_{\bm{\theta}}(\cdot\mid\bm{\tau})}
\left[
r(\bm{\tau},y)
\right].
\end{equation}
In practice, we use GRPO \citep{shao2024deepseekmath}, which replaces the value
network with a group-normalized advantage. For responses generated from the same
prompt, the advantage is
\begin{equation}
\hat{A}_i
=
\frac{
r_i-\mathrm{mean}(\{r_i\}_{i=1}^{K})
}{
\mathrm{std}(\{r_i\}_{i=1}^{K})
},
\end{equation}
and the policy is updated with a clipped importance-ratio objective regularized
by the KL divergence to a reference policy.

\paragraph{Online prompt selection.}
Prompt selection is needed because exact rollout-based evaluation is expensive.
Evaluation-based methods over-sample a candidate pool
$\mathcal{T}_{t}^{\hat{B}}$ and keep prompts with non-degenerate reward signals,
for example those satisfying
$\mathrm{std}(\{r_i^{\bm{\tau}}\}_{i=1}^{K})>0$. Such filtering is effective but
requires extra generations. Prediction-based methods reduce this cost by
estimating prompt difficulty, typically through a predictive model of
$p(\rho_t^{\bm{\tau}}\mid \bm{\tau},\bm{\theta}_t)$ or its history-conditioned
variant. However, prompt-wise independent estimators may fail to share
information across prompts and may lag behind the changing policy.

\subsection{Unified Predictive Selection View}
\label{app:unified_selection}

Both robust task sampling and active prompt selection can be expressed as
selection from predicted outcomes. Let
$\hat{y}^{t}_{\bm{\tau}}$ denote the model-predicted outcome for candidate
$\bm{\tau}$. For active prompt selection, this outcome is the predicted success
rate, $\hat{y}^{t}_{\bm{\tau}}=\hat{\rho}^{t}_{\bm{\tau}}$, and the target is an
intermediate difficulty level $\rho^*$. For robust task sampling, the outcome is
the predicted risk, $\hat{y}^{t}_{\bm{\tau}}=\hat{\ell}_{t}(\bm{\tau})$, and the
target is the largest predicted risk in the candidate pool. The selected batch is
\begin{equation}
\label{eq:unified_selection_appendix}
\mathcal{T}_{t}^{B}
=
\operatorname*{arg\,top-}B_{\bm{\tau} \in \mathcal{T}_{t}^{\hat{B}}}
\left(
-\big\| \hat{y}^{t}_{\bm{\tau}} - c_t^* \big\|_2^2
\right),
\end{equation}
where $c_t^*=\rho^*$ for prompt selection and
$c_t^*=\max_{\bm{\tau}\in\mathcal{T}_{t}^{\hat{B}}}
\hat{\ell}_{t}(\bm{\tau})$ for robust task sampling.

\section{Risk-Aware Evaluation and Predictive Ranking Metrics}
\label{app:risk_metrics}

This section summarizes the metrics and regularity conditions used to evaluate
tail-task robustness and risk-prediction quality. We write the task risk induced
by a model parameterized by $\bm{\theta}$ as
\[
\ell_{\bm{\theta}}(\bm{\tau})
:=
\ell(\mathcal{D}_{\bm{\tau}}^{Q},
\mathcal{D}_{\bm{\tau}}^{S};\bm{\theta}),
\]
where $\bm{\tau}\sim p(\bm{\tau})$ is a task identifier. Depending on the
benchmark, $\ell_{\bm{\theta}}(\bm{\tau})$ is instantiated as MSE for regression,
negative return for control, or failure probability for RL post-training.

\subsection{\texorpdfstring{$\mathrm{CVaR}_{\alpha}$}{CVaR} for Tail-Task Robustness}
\label{app:cvar_metric}

\begin{definition}[\texorpdfstring{$\mathrm{CVaR}_{\alpha}$}{CVaR}]
\label{def:cvar}
Let
$
F_{\bm{\theta}}(z)
=
\Pr_{\bm{\tau}\sim p(\bm{\tau})}
\left(
\ell_{\bm{\theta}}(\bm{\tau}) \le z
\right)
$
be the cumulative distribution function of the task risk, and define its
$\alpha$-quantile as
$
q_{\alpha}(\bm{\theta})
=
\inf\{z:F_{\bm{\theta}}(z)\ge \alpha\}.
$ Following~\citep{rockafellar2000optimization}, the $\alpha$-level Conditional
Value-at-Risk is the expected risk over the upper tail:
\begin{equation}
\label{eq:cvar_def_app}
\mathrm{CVaR}_{\alpha}(\bm{\theta})
=
\mathbb{E}
\left[
\ell_{\bm{\theta}}(\bm{\tau})
\;\middle|\;
\ell_{\bm{\theta}}(\bm{\tau})\ge q_{\alpha}(\bm{\theta})
\right].
\end{equation}
Equivalently, for a continuous risk distribution,
\begin{equation}
\label{eq:cvar_tail_dist_app}
\mathrm{CVaR}_{\alpha}(\bm{\theta})
=
\int z\,dF_{\bm{\theta}}^{\alpha}(z),
\qquad
F_{\bm{\theta}}^{\alpha}(z)
=
\begin{cases}
0, & z<q_{\alpha}(\bm{\theta}), \\[1mm]
\dfrac{F_{\bm{\theta}}(z)-\alpha}{1-\alpha},
& z\ge q_{\alpha}(\bm{\theta}),
\end{cases}
\end{equation}
where $F_{\bm{\theta}}^{\alpha}$ is the normalized tail-risk distribution. We
denote the corresponding tail-task density by
$p_{\alpha}(\bm{\tau};\bm{\theta})$.
\end{definition}

Under this convention, $\mathrm{CVaR}_{0.9}$ evaluates the expected risk on the
hardest $10\%$ of tasks. Thus, lower CVaR indicates stronger tail-task robustness.
For control tasks, where risk is defined as negative return, reporting
$-\mathrm{CVaR}_{\alpha}$ gives the corresponding CVaR return, for which higher
values are better.

\subsection{Spearman Rank Correlation for Risk Prediction}
\label{app:src_metric}

While $\mathrm{CVaR}_{\alpha}$ measures the robustness of the learned policy, it
does not directly assess whether the risk-predictive model correctly identifies
hard tasks. Since predictive samplers mainly rely on the relative ordering of
candidate tasks, we additionally report Spearman Rank Correlation (SRC) between
predicted and observed risks.

Given a candidate or validation set $\{\bm{\tau}_i\}_{i=1}^{N}$, let
$\hat{\ell}_i=\hat{\ell}(\bm{\tau}_i)$ be the predicted risk and
$
\ell_i
=
\ell(\mathcal{D}_{\bm{\tau}_i}^{Q},
\mathcal{D}_{\bm{\tau}_i}^{S};\bm{\theta})
$
be the observed risk. Denoting the rank operator by $\operatorname{rk}(\cdot)$,
SRC is computed as
\begin{equation}
\label{eq:src_app}
\mathrm{SRC}
=
1
-
\frac{
6\sum_{i=1}^{N}
\left(
\operatorname{rk}(\hat{\ell}_i)
-
\operatorname{rk}(\ell_i)
\right)^2
}{
N(N^2-1)
}.
\end{equation}

SRC lies in $[-1,1]$. Values close to $1$ indicate that the predicted risk order
matches the observed task-difficulty order. This metric is particularly relevant
for predictive task sampling: even if the absolute risk values are not perfectly
calibrated, a high SRC implies that the sampler can still prioritize informative
or high-risk tasks. In contrast, low SRC indicates rank-flipping in the predicted
difficulty landscape, which may lead to suboptimal task selection and weaker
tail-task performance.

\section{Experiment Details}
\label{expdetails}
For all baselines, we follow the default configurations reported in their original papers whenever available, including model architectures, optimizer settings, task-selection rules, and hyperparameters. When a baseline is adapted to our shared MPTS-style evaluation protocol, we keep the training budget, candidate-pool size, batch size, number of seeds, and evaluation schedule consistent with ICRPM for fair comparison.
\subsection{Meta-Reinforcement Learning Experiments}
\label{app:metarl}

\paragraph{Shared Setup}
We evaluate on three MuJoCo Meta-RL benchmarks (\textbf{HalfCheetahVel},
\textbf{Walker2dVel}, \textbf{ReacherPos}) from the MAML task-distribution suite.
Tasks are parameterised by continuous descriptors $\boldsymbol{\tau}$
(target velocity or goal position) sampled from $p(\bm{\tau})$ at each iteration.
All methods share the same meta-learner: \textbf{MAML} with TRPO as the inner
optimiser.  Key shared hyperparameters are listed in Table~\ref{tab:metarl_hparams}.

\begin{table}[h]
\centering
\caption{Shared meta-learning hyperparameters (Meta-RL).}
\label{tab:metarl_hparams}
\small
\begin{tabular}{lc}
\toprule
Hyperparameter & Value \\
\midrule
Meta-training iterations              & 500 \\
Meta-batch size $B$                   & 20 tasks \\
Validation batch size                 & 20 tasks (uniform) \\
Inner adaptation steps                & 1 \\
Fast batch size (trajectories / task) & 20 \\
Max KL divergence                     & 0.01 \\
Discount factor                       & 0.99 \\
GAE $\lambda$                         & 1.0 \\
Random seeds                          & 10 \\
\bottomrule
\end{tabular}
\end{table}

\paragraph{Evaluation protocol.}
All methods are evaluated on a held-out validation set of 20 tasks sampled uniformly
from $p(\bm{\tau})$ at each iteration; validation episodes are excluded from the
environment interaction count.

\paragraph{ICRPM Configuration for Meta-RL}

ICRPM uses Transformer backbone RPM that learns a predictive distribution
$p(\ell\vert\boldsymbol{\tau},C_t)$ conditioned on the latest context set
$C_t = \{(\boldsymbol{\tau}_{t,i},\ell_t(\boldsymbol{\tau}_{t,i}))\}_{i=1}^{B}$.

\paragraph{Architecture Details.}
The encoder maps context pairs
$\{(\boldsymbol{\tau}_i,\ell_i)\}_{i\in C_t}$ through a depth-4 MLP embedding,
a Transformer encoder (self-attention over context), and a cross-attention module
where each target task $\boldsymbol{\tau}$ attends to the encoded context.
The decoder outputs Gaussian parameters $(\mu,\sigma)$ with
$\sigma = 0.1+0.9\cdot\mathrm{softplus}(\cdot)$.
The transformer-based ICRPM is parameter-matched to the vanilla RPM for a fair comparison
($\approx$0.65\,M parameters).

\paragraph{Training objective.}
At each iteration $t$, one gradient step is taken on the weighted NLL objective
(Eq.~(\ref{eq:icrpm_obj}) in the main text):
\begin{equation}
  \min_{\boldsymbol{\phi}}\;\mathcal{J}(\boldsymbol{\phi})
  = -\sum_{k=1}^{W}\gamma^{k}
    \left(\frac{1}{B}\sum_{i=1}^{B}
      \ln p_{\boldsymbol{\phi}}\!\left(
        \ell_{t-k+1}(\bm{\tau}_{t-k+1,i})
        \vert
        \boldsymbol{\tau}_{t-k+1,i},\,H_{t-k}
      \right)
    \right),
  \label{eq:icrpm_obj_app}
\end{equation}
where $H_{t-k} = \{(\boldsymbol{\tau}_{t-k, i},\ell_{t-k, i})\}_{i=1}^{B}$
is the per-step history at time $t-k$, $W=3$ is the sliding-window size,
$\gamma=0.9$ is the temporal decay, and Adam (lr\,$=5\times10^{-3}$) is used.

\paragraph{Task selection.}
This retains the predictive selection setup in MPTS~\citep{wang2025modelpredictivetasksampling}. At each
iteration, $1.5B$ candidate tasks are drawn from $p(\tau)$. ICRPM predicts the
mean risk $\hat y_i=\mathbb E_{p_\phi(y|\hat\tau_i,H_t)}[y]$ for each candidate.
For robust task sampling, we set
$c^*=\max_i \hat y_i$ and select the $B$ candidates with the largest score
$-\|\hat y_i-c^*\|_2^2$, which is equivalent to selecting the top-$B$ candidates
by predicted risk in the scalar-risk case.

\subsection{RLVR Countdown Experiment}
\label{app:rlvr}

\paragraph{Setup}

We use the \textbf{Countdown} mathematical reasoning task (20k training instances,
Countdown-3to4) with \textbf{Qwen3-4B-Base} as the policy backbone
$\pi_{\boldsymbol{\theta}}$, trained via \textbf{GRPO+} inside the \textbf{veRL}
framework (v0.3.1) on 4$\times$A100-80\,GB GPUs.  Key training hyperparameters
are listed in Table~\ref{tab:rlvr_hparams}.

\begin{table}[h]
\centering
\caption{RLVR training hyperparameters.}
\label{tab:rlvr_hparams}
\small
\begin{tabular}{lc}
\toprule
Hyperparameter & Value \\
\midrule
Base model                         & Qwen3-4B-Base \\
Training steps                     & 100 \\
Prompt batch size $B$              & 128 (4-GPU) \\
Responses per prompt $K$           & 8 \\
Max response length                & 1,024 tokens \\
Policy optimiser                   & AdamW, lr $1\times10^{-6}$ \\
Entropy coefficient                & 0.001 \\
KL coefficient                     & 0.0 \\
PPO clip ratio (low\,/\,high)      & 0.2\,/\,0.28 \\
\bottomrule
\end{tabular}
\end{table}

The reward function follows the format-aware Countdown setting:
$r=1$ (correct), $r=0.1$ (incorrect but well-formatted), $r=0$ (otherwise).
Evaluation reports pass@1 accuracy on held-out CD34 and the harder CD4 benchmark.

\paragraph{ICRPM Configuration for RLVR}

ICRPM replaces the prompt-difficulty predictor (RPM).
At each training step, rollouts are generated by vLLM using the current actor.
After reward computation, prompt embeddings
$\boldsymbol{\tau}\in\mathbb{R}^{1024}$ are extracted from the current actor's
last hidden states and paired with observed solve rates to form the context
$H_t$, , where the solve rate is the fraction of generated responses that exactly solve the Countdown instance. These task representations are not frozen; their policy-induced drift is
treated as part of the non-stationary task-risk context modeled by ICRPM.
The ICRPM is trained with the objective in Eq.~\eqref{eq:icrpm_obj_app}, using
$W=3$, $\gamma=0.9$, and 20 gradient steps per training step, matching the
vanilla RPM. In RLVR, ICRPM is scaled to match the vanilla RPM parameter count
($\approx$42.8\,M). The task-sampling ratio is 80\%.

\subsection{Domain Randomization Experiments}
\label{app:dr}

For zero-shot DR we follow the MPTS protocol with \textbf{TD3} as the policy
backbone on \textbf{Ergo-Reacher} and
\textbf{Lunar-Lander}~\citep{mehta2020active}.
Each task $\boldsymbol{\tau}$ is an environment instance induced by a continuous
domain-randomisation identifier; the risk
$\ell(\boldsymbol{\tau}) = L(\boldsymbol{\theta}_t,\boldsymbol{\tau})$ is the
negative episodic return without inner-loop adaptation.
The same task-batch size $B$, candidate-pool size $\hat{B}$, and baseline protocol
as Meta-RL are used. ICRPM uses a parameter-matched Transformer architecture and the same temporal augmentation settings ($W=3,\gamma=0.9$).

\subsection{Few-Shot Sinusoid Regression}
\label{app:sinusoid}

Each task is specified by $\boldsymbol{\tau}=(a,b)$ with
$f_{\boldsymbol{\tau}}(x)=a\sin(x-b)$, $a\in[0.1,5.0]$, $b\in[0,\pi]$.
The meta-learner is MAML and task risk is query-set MSE after adaptation.
We report average MSE and CVaR$_\alpha$ (Eq.~\eqref{eq:cvar_tail_dist_app})
with $\alpha\in\{0.5,0.7,0.9\}$, along with SRC (Eq.~\eqref{eq:src_app}) throughout
training. ICRPM is parameter-matched to the vanilla RPM and uses the same temporal
augmentation settings ($W=3$, $\gamma=0.9$).

\subsection{Ablation Studies}
\label{app:ablation}

Ablations are conducted on \textbf{Lunar-Lander DR} with 10 independent runs each. The ICRPM parameter count is equal to the vanilla 
RPM across all ablation studies for fair comparison.
We vary four design dimensions: acquisition principle, RPM backbone, history augmentation, and the window/decay hyperparameters.

\begin{enumerate}
  \item \textbf{Acquisition principle.}
    We compare worst-case selection by predicted risk $\mu(\ell)$ with a diversity-aware
    variant that combines $\mu(\ell)$ with a set-level diversity bonus.
    Both variants are evaluated under identical optimisation budgets and compared with
    PDTS~\citep{qu2025fast}. This ablation tests whether the improvement comes from the
    risk predictor itself rather than a specific acquisition heuristic.

  \item \textbf{RPM backbone and history augmentation.}
    We compare the default Transformer backbone with a parameter-matched MLP that
    retains the same context-conditioned look-ahead objective in Eq.~\eqref{eq:icrpm_obj_app}.
    To isolate the role of history reuse from look-ahead prediction, we also evaluate
    \emph{MPTS-augment}, which augments vanilla MPTS~\citep{wang2025modelpredictivetasksampling} with the same three-step history window
    and temporal decay $\gamma=0.9$ used by ICRPM, but without the context-conditioned
    one-step extrapolative objective. This comparison separates the effect of the backbone,
    the sliding-window augmentation, and the proposed look-ahead formulation.

  \item \textbf{Look-back window $W$.}
    We vary $W\in\{1,2,3,8\}$ with $\gamma=0.9$ fixed. Smaller windows use fresher
    histories but provide less task coverage, while larger windows improve sample reuse
    but can introduce stale context from earlier policy states. We set $W=3$ as the default
    trade-off between coverage and freshness.

  \item \textbf{Temporal decay factor $\gamma$.}
    We vary $\gamma\in\{0.5,0.7,0.9,1.0\}$ with the default window size fixed. Smaller
    values emphasize only the most recent batches, whereas $\gamma=1.0$ weights stale and
    fresh histories equally. Among the tested values, $\gamma=0.9$ provides a stable balance by mildly discounting older histories, and we therefore use it as the default.
\end{enumerate}

All ablation settings report CVaR$_{0.9}$ return, average return, and SRC
(Eq.~\eqref{eq:src_app}) on the task-distribution validation set. The results show that
ICRPM remains effective across acquisition rules and backbone choices, while MPTS-augment
does not match ICRPM under the same history-window and decay configuration. This indicates
that the gains are not explained by history reuse alone, but mainly come from
context-conditioned one-step look-ahead risk prediction.

\subsection{Compute resources.}
\label{appendix:compute}
Table~\ref{tab:compute} reports approximate single-run wall-clock times, where each entry
corresponds to one method on one benchmark/environment with one random seed, including
training and periodic evaluation but excluding preliminary debugging runs. For Sinusoid and
continuous-control experiments, ERM~\citep{vapnik1998statistical}, MPTS~\citep{wang2025modelpredictivetasksampling}, GDRM~\citep{sagawa2019distributionally}, OHTM~\citep{kumar2023effect}, and ICRPM have comparable
runtime, while DRM is about 30--40\% slower due to its additional importance-weight
computation. For RLVR Countdown, Uniform, PCL~\citep{gao2026prompt}, MoPPS~\citep{qu2026can}, and ICRPM have comparable
runtime, while GRESO~\citep{zheng2025act} is about 60\% slower due to additional rollout-selection overhead.
\begin{table}[h]
\centering
\small
\caption{Approximate single-run compute resources.}
\label{tab:compute}
\begin{tabular}{llll}
\toprule
Experiment & Hardware & Common & Slower \\
\midrule
Sinusoid & RTX 4090 24GB & $\sim$12.2 min & DRM: $\sim$16.4 min \\
DR & RTX 4090 24GB & $\sim$3.0 h & DRM: $\sim$4.2 h \\
Meta-RL & A10 24GB & $\sim$8.7 h & DRM: $\sim$11.6 h \\
RLVR & 4$\times$A100 80GB & $\sim$11.1 h & GRESO: $\sim$17.8 h \\
\bottomrule
\end{tabular}
\end{table}

\subsection{Assets, Licenses, and Terms of Use.}
All experiments use publicly available research assets. For LLM post-training, we use Qwen3-4B-Base as the policy backbone, and implement training and inference with veRL and vLLM. These model and software assets are released under the Apache License 2.0. For continuous-control experiments, we use OpenAI Gym with MuJoCo-based environments through the legacy \texttt{mujoco-py} interface. OpenAI Gym and \texttt{mujoco-py} are released under permissive open-source licenses, while the MuJoCo simulator is used subject to its applicable license terms; we do not redistribute MuJoCo binaries, license keys, or third-party assets. Our TD3 and MAML implementations are adapted from or benchmarked against publicly available research codebases released under permissive licenses. For Countdown, we use public arithmetic-reasoning datasets from Hugging Face for research evaluation only, do not redistribute raw dataset files, and provide sufficient preprocessing and evaluation details to support reproducibility. We comply with the terms of use of all upstream assets and retain the corresponding license and copyright notices where applicable.

\section{Limitations}
\label{app:limitations}

ICRPM is proposed as a context-conditioned extrapolative RPM for improving risk approximation in MPTS-style task sampling. Several aspects remain open for future investigation.

\begin{enumerate}
    \item \textbf{History-based extrapolation.}
    ICRPM relies on recent optimization histories to infer next-step task risk.
    This is effective when recent histories capture local risk drift, but may be less reliable under abrupt policy updates, highly noisy rewards, or task regions poorly covered by the history buffer.

    \item \textbf{Scope of validation and analysis.}
    Our theory uses standard boundedness and smoothness assumptions to characterize temporal mismatch, and our experiments cover regression, robust continuous control, and RLVR-based LLM post-training.
    Extending the analysis and evaluation to larger models, longer horizons, and broader task distributions remains an important direction for future work.
\end{enumerate}

\section{Broader Impacts}
ICRPM improves sample-efficient task selection by reducing temporal-lag bias in online risk prediction. This can lower interaction and compute costs and improve robustness on difficult or tail tasks in simulated control and verifiable-reward reasoning settings. However, more efficient task selection may also reduce the cost of RL-based post-training for capable models, including applications with potential misuse. Moreover, learned risk predictors may amplify blind spots if the task distribution is biased or incomplete. We therefore recommend using such methods with broad task coverage, safety-oriented stress tests, and task-distribution audits before any real-world deployment. Our experiments use simulated environments and algorithmic reasoning benchmarks only, and do not involve human subjects, private user data, or real-world deployment.

\end{document}